\documentclass[sigconf]{acmart}
\usepackage{microtype}
\usepackage{etoolbox}

\usepackage{amsmath,eqnarray,mathtools}
\usepackage{amsfonts,bbm,bm,courier}
\usepackage{setspace,relsize}
\usepackage{graphicx,placeins}
\usepackage{xcolor}
\usepackage{enumitem}
\usepackage{caption}
\usepackage{hyperref}
\hypersetup{%
  colorlinks=true,
  urlcolor=blue,%
  linkcolor=blue,%
  citecolor=blue,%
  linkbordercolor=blue,
  pdfborderstyle={/S/U/W 1}
}
\usepackage{subcaption,afterpage}

\usepackage{todonotes}

\DeclareMathOperator*{\argmin}{argmin}
\DeclareMathOperator*{\argmax}{argmax}

\newcommand{\indic}[1]{\mathbbm{1}[#1]}

\newcommand{\R}{\mathbb{R}}

\newcommand{\Simplex}[1]{\mathbb{S}^{#1}}

\newcommand{\data}{\mathcal{D}}
\newcommand{\X}{\mathcal{X}}

\newcommand{\yhat}[1]{\hat{y}^{(#1)}}
\newcommand{\wb}{w}

\newcommand{\st}{\textnormal{s.t.}}
\newcommand{\miprange}[3]{{#1}={#2},...,{#3}}

\newcommand{\clf}[0]{h}
\newcommand{\ambiguity}[4]{A_{{#1},{#2}}({#3};{#4})}

\newcommand{\Hset}[0]{\mathcal{H}}
\newcommand{\epsset}[1]{\mathcal{R}_\epsilon({#1})}
\newcommand{\alphaset}[0]{\Hset_\alpha}

\usepackage{algorithm,algpseudocode}

\algnewcommand\algorithmicinput{\textbf{Input}}

\algnewcommand\algorithmicinitialize{\textbf{Initialize}}
\algnewcommand\algorithmicbigstep{\textbf{Step}}
\algnewcommand\INPUT{\item[\algorithmicinput]}
\algnewcommand\INITIALIZE{\item[\algorithmicinitialize]}
\algnewcommand{\STEP}[1]{\item[\algorithmicbigstep]{\textbf{#1}}}
\algnewcommand{\InputExplanation}[2][.6\linewidth]{\leavevmode\hfill\makebox[#1][r]{~{\footnotesize{#2}}}}
\algnewcommand{\InitializationExplanation}[2][.6\linewidth]{\leavevmode\hfill\makebox[#1][r]{~{\footnotesize{#2}}}}
\algnewcommand{\alginput}[2]{\INPUT{#1}\InputExplanation{#2}}
\algnewcommand{\StateComment}[2]{\State{#1}\InputExplanation{#2}}
\algnewcommand{\alginitialize}[2]{\Statex{#1}\InitializationExplanation{#2}}
\algrenewcommand\algorithmiccomment[2][]{#1\hfill\textit{\scriptsize{#2}}}

\newcommand{\thankssymb}[1]{\textsuperscript{#1}}

\AtBeginDocument{%
  \providecommand\BibTeX{{%
    \normalfont B\kern-0.5em{\scshape i\kern-0.25em b}\kern-0.8em\TeX}}}

\copyrightyear{2023}
\acmYear{2023}
\setcopyright{acmlicensed}\acmConference[FAccT '23]{2023 ACM Conference on Fairness, Accountability, and Transparency}{June 12--15, 2023}{Chicago, IL, USA}
\acmBooktitle{2023 ACM Conference on Fairness, Accountability, and Transparency (FAccT '23), June 12--15, 2023, Chicago, IL, USA}
\acmPrice{15.00}
\acmDOI{10.1145/3593013.3593998}
\acmISBN{979-8-4007-0192-4/23/06}

\begin{document}

\title[Multi-Target Multiplicity]{Multi-Target Multiplicity: Flexibility and Fairness in Target Specification under Resource Constraints}

\author{Jamelle Watson-Daniels}
\email{jwatsondaniels@g.harvard.edu}
\affiliation{%
  \institution{Harvard University\thankssymb{*}}
  \streetaddress{}
  \city{Cambridge}
  \state{MA}
  \country{USA}
  \postcode{}
  \thanks{\thankssymb{*}Work done as an intern at Microsoft Research.}
}

\author{Solon Barocas}
\email{solon@microsoft.com}
\affiliation{%
 \institution{Microsoft Research}
 \streetaddress{}
 \city{New York}
 \state{New York}
 \country{USA}}

\author{Jake M. Hofman}
\email{jmh@microsoft.com}
\affiliation{%
  \institution{Microsoft Research}
  \city{New York}
  \state{New York}
  \country{USA}
}

\author{Alexandra Chouldechova}
\email{alexandrac@microsoft.com}
\affiliation{%
  \institution{Microsoft Research}
  \streetaddress{}
  \city{New York}
  \state{New York}
  \country{USA}}

\renewcommand{\shortauthors}{Watson-Daniels, et al.}

\begin{abstract}

Prediction models have been widely adopted as the basis for decision-making in domains as diverse as employment, education, lending, and health. Yet, few real world problems readily present themselves as precisely formulated prediction tasks. In particular, there are often many reasonable target variable options.  Prior work has argued that this is an important and sometimes underappreciated choice, and has also shown that target choice can have a significant impact on the fairness of the resulting model.  However, the existing literature does not offer a formal framework for characterizing the extent to which target choice matters in a particular task. Our work fills this gap by drawing connections between the problem of target choice and recent work on predictive multiplicity.  Specifically, we introduce a conceptual and computational framework for assessing how the choice of target affects individuals’ outcomes and selection rate disparities across groups.  We call this multi-target multiplicity.   Along the way, we refine the study of single-target multiplicity by introducing notions of multiplicity that respect resource constraints---a feature of many real-world tasks that isn’t captured by existing notions of predictive multiplicity.  We apply our methods on a healthcare dataset, and show that the level of multiplicity that stems from target variable choice can be greater than that stemming from nearly-optimal models of a single target.  

\end{abstract}



\keywords{target specification, fairness, multiplicity, predictive inconsistency}



\newcommand{\ydiff}{\tilde{y}}
\newcommand{\yobs}[1]{\ydiff^{(#1)}}
\newcommand{\err}[0]{\mathcal{E}}
\newcommand{\Ydiff}{\tilde{Y}}
\newcommand{\Yobs}[1]{\Ydiff^{(#1)}}

\newcommand{\onept}{i}
\newcommand{\twopt}{{i'}}
\newcommand{\topK}{\kappa}
\newcommand{\predIndic}{I}
\newcommand{\combine}{c}

\newcount\Comments  
\Comments = 1
\newcommand{\kibitz}[2]{\ifnum\Comments=1{\color{#1}{#2}}\fi}
\newcommand{\alex}[1]{\kibitz{purple}{AC: #1}}
\newcommand{\solon}[1]{\kibitz{orange}{SB: #1}}
\newcommand{\jake}[1]{\kibitz{blue}{JH: #1}}
\newcommand{\jwd}[1]{\kibitz{magenta}{JWD: #1}}

\maketitle

\section{Introduction}

Scholars have argued that prediction problems are ubiquitous across many domains of decision-making, from employment, to education, to health \cite{kleinberg2015prediction}. Yet real-world problems rarely present themselves as fully formed machine learning tasks \cite{provost2013data}. 
Critically, it is often not clear what target should be predicted to help decision makers achieve their goals %
\cite{hand1994deconstructing, Passi2019}. It is far from obvious, for example, how employers should go about making such choices in their hiring practices: if the goal is to hire the ``best'' people, what exactly should the model be predicting \cite{barocas2016big, Passi2019, kim2022race}? For a sales position, employers might choose to predict annual sales figures. But they could alternatively choose to predict how well the applicant will work with others, whether customers will actually enjoy interacting with the applicant, etc.  %
Even in domains where target choice might seem more obvious, there can still be a good deal of uncertainty. For example, while it might seem self-evident that creditors should be predicting default, what constitutes ``default'' is not a given. Creditors need to make an affirmative choice about the number of months of missed payments that ultimately count as ``default'' \cite{10.1214/088342306000000060}. 
In some cases, the decision is not based on just one chosen target, but instead a combination of targets. For example, many algorithmic tools currently used in criminal justice and human services function by aggregating predictions of several different targets, ranging from different types of criminal justice system encounters, to mental and physical health outcomes, to measures of housing stability \cite{kithulgoda2022predictive, arnold_ventures_2022}.

A recent line of work has explored the implications of this flexibility in target variable choice for fairness. In particular, researchers have pointed out that different choices for the prediction target can lead to more or less disparity in selection rates across groups \cite{Passi2019, Obermeyer2019, 10.1145/3442188.3445901, mitchell2021algorithmic, mullainathan2021inequity, martin2020participatory, kim2022race, fazelpour2021algorithmic}. One particularly well-known study by \citet{Obermeyer2019} illustrates both the risks and benefits of target choice. The authors examine an algorithm developed by a healthcare system used in determining patient eligibility for a high-risk coordinated care management program. They find that the healthcare system's choice to adopt healthcare costs as the target of prediction led to notable and \textit{avoidable} racial disparities.  Because Black patients in the United States generally incur lower health care costs at equal levels of underlying health care needs, predicting costs results in a score that systematically prioritizes healthier White patients over less healthy Black patients.  The authors show that a good deal of the racial disparity could have been avoided had the healthcare system instead chosen to predict a more direct measure of health outcomes. The study has been received as a important lesson in the dangers of insufficiently careful target choice.
But it also highlights that practitioners can take advantage of the latitude afforded by target choice to reduce selection rate disparities. 

The existing literature offers several examples that highlight the potential importance of target variable choice.  Prior work does not, however, offer a more general mathematical or computational framework for characterizing the extent to which target variable choice affects individuals' outcomes and selection rate disparities across groups. %
Our work aims to fill this gap.  Specifically, we draw connections to recent work on predictive multiplicity, which characterizes the flexibility in individual predictions over the set of ``good models,''  defined as models that achieve near-equivalent performance \cite{Marx2019, dong2019variable}.  This line of work demonstrates how such models may differ considerably in their individual predictions, and proposes computational methods for quantifying this variability through multiplicity measures such as \textit{ambiguity} and \textit{discrepancy}.  These observations, which are closely connected to the ``Rashomon effect'' introduced by \citet{Breiman2001}, have been of particular interest to those working on issues of fairness because they suggests that there will often exist a model with comparable performance but with less disparity in selection rates across groups than some baseline model---i.e., that the accuracy-fairness trade-off may be minimal in many settings \cite{Marx2019, Coston2021, Watson-Daniels2022, black2022model, Ali2021}.  By analogy, in the motivating "multi-target" setting---where there are many possible reasonable prediction targets to choose from---we can consider the set of "good models" that arises from, say, models that predict \textit{any} of the  individual targets well, %
models that predict a \textit{combination} of targets well, or that \textit{combine the predictions} of single-target models.  If each of the possible targets individually has merit as a basis for predictive allocation, then any \textit{combination} of those predictions is arguably also a reasonable candidate model.  We formalize these ideas and provide examples in the technical sections of the paper.

Prior work on multiplicity has at times explicitly steered clear of viewing target choice as a source of multiplicity, while at the same time acknowledging that it plays an important role in problem formulation more generally~\cite{black2022model}.  One reason for this is that multiplicity has been studied with respect to the task of predicting a pre-specified target. %
From this perspective, considering different outcomes amounts to considering a different task.  In our setting, however, we adopt a broader view of the task for which multiplicity is being assessed.  Specifically, we note that all the motivating examples we have just discussed can be thought of as \textit{predictive allocation tasks}---tasks where historical data is used to learn a ``prioritization'' or ``risk'' score and where that score then serves as the basis for deciding how to allocate resources, usually as part of human-in-the-loop decision processes. If a range of models trained to predict different targets can be similarly helpful in performing a predictive allocation task, then it is reasonable to understand the flexibility afforded by target choice in terms of multiplicity as well. %

In addition to introducing a framework for multi-target multiplicity, we also refine the standard treatment of predictive multiplicity in the single-target setting to account for additional practical constraints inherent in predictive allocation tasks.  In practice, predictive allocation tasks are typically governed by resource constraints.  There is only a finite amount of benefit, burden, or scrutiny that the system is able to allocate. 
For instance, the algorithm investigated by Obermeyer et al. was developed to help allocate coordinated care management to a certain number of clients. Similarly, employers cannot offer jobs to \emph{everyone} they predict will be a sufficiently good employee, whatever target or set of targets they choose to predict to make such an assessment. Given their limited budgets, they are likely only able to offer jobs to a select few applicants.  This means that the set of ``good models'', whether in the single-target or multi-target setting, can only include models that satisfy the resource constraint.  In this work we demonstrate how to introduce resource constraints into the study of multiplicity.

In summary, our work introduces the concept of multi-target multiplicity, and provides a formal and computational framework for quantifying the level of multiplicity that exists in a given predictive allocation task.  Along the way we introduce a refinement of single-target predictive multiplicity to the resource constrained setting, and introduce corresponding computational methods.  Our primary contributions are as follows.  
\begin{enumerate}[leftmargin=*] 
    \item We introduce a framework for assessing single-target multiplicity in the presence of resource constraints (\textsection\ref{sec::multiplicity}).  We define a new measure of predictive multiplicity (top-$\kappa$ ambiguity) and present a mixed integer program (MIP) to calculate this ambiguity measure for linear models.  
    \item We introduce the concept of multi-target multiplicity alongside a framework for assessing multi-target multiplicity for predictive allocation tasks (\textsection\ref{sec:multiple-target-ambiguity}).  
    \item We demonstrate how the framework can be used to assess fairness-related measures by presenting a MIP that calculates the minimum and maximum attainable selection rate for a given group (\textsection\ref{sec::multi_fair}).
    \item We demonstrate our framework on the healthcare dataset released by Obermeyer et al. and provide semi-synthetic experiments that aim to clarify how we might be able to improve fairness by moving to a  multi-target setting (\textsection\ref{sec:eval}).  Our results show that the level of multiplicity that stems from target variable choice can be considerably greater than that stemming from nearly-optimal models of a single target.
    
\end{enumerate}

\section{Related Work}
\label{sec:related-work}

\emph{Problem formulation and fairness.} 
Scholars have identified various reasons why the choice of target might raise concerns with fairness: some outcomes or qualities of interests might just be more evenly distributed across the population than others \cite{Passi2019, kim2022race}; certain outcomes or qualities of interests might be easier to predict with similar degrees of accuracy across the population than others \cite{coston2022validity}; some kinds of selection bias might cause certain outcomes or qualities of interest to be observed more or less frequently in certain groups rather than others, even if they occur at similar rates in reality \cite{lum2016predict}; certain targets might suffer from more so-called ``label bias'' than others---that is, systematically less accurate observations of the true value of the target for members of some groups than others \cite{jiang2020identifying, Coston2021, 10.1145/3442188.3445901}. Indeed, one way to understand the Obermeyer et al. study is as a form of label bias since healthcare costs acted as a systematically inaccurate measure of underlying healthcare needs. Our work departs from much of this literature by focusing on cases where there is no obviously right or clearly preferable choice of target or proxy and thus uncertainty about which to choose or whether to choose more than one.

\emph{Multi-task learning, multi-criteria decision-making, latent variable modeling, and fairness.}  While our use of the term "multi-target" might suggest a close connection to fairness considerations in multi-task learning (see, e.g., \cite{multitaskfairness}), the problem we study is distinct.  Whereas in multi-task learning the goal is to perform well on (and assess fairness for) \textit{each} of $K$ prediction tasks by borrowing strength across tasks, in our setting we are interested in arriving at a \textit{single model}, which may not perform optimally on any individual task, but which successfully captures multiple desiderata.  In this sense, our setting is more closely related to recent work on latent variable modeling in recommender systems that aims to optimize for a latent \textit{value} using a combination of noisy observed measures, such as clicks, replies, reshares, and other observable forms of user engagement \cite{milli2021optimizing, kleinberg2022challenge}.  A key difference is that we do not posit a specific notion of optimality, and instead explore the degree of multiplicity inherent in a class of learning procedures for forming a univariate prediction from multiple available targets.  Lastly, our work connects to the extensive literature on multi-criteria decision-making (MCDM) in operations research.  Indeed, the index model and index variable approaches we introduce in \textsection\ref{sec:multiple-target-ambiguity} parallel the classic \textit{weighted sums} method of combining multiple criteria (e.g., loss or other objective functions) into a single objective \cite{gass1955computational}. However, whereas the focus of MCDM is in the values of the different objective functions, we examine multiplicity, which pertains to the variability in prediction decisions for \textit{individual people or cases}.

\emph{Predictive multiplicity and fairness.} There is also a growing literature that seeks to explore the normative implications of multiplicity. Scholars have investigated the degree to which multiplicity can be leveraged to improve interpretability \cite{Semenova2019a} and explainability \cite{Fisher2019, Dong2019, Pawelczyk2020}. Others have examined the danger multiplicity poses for robustness \cite{DAmour2020} and non-arbitrariness \cite{Black2021Leave-one-outUnfairness, black2022model, ReconcilingProbabilityForecasts, rashomoncapacity, cooper2023variance}. Still others have focused on its implications for fairness \cite{Marx2019, Ali2021, Coston2021, black2022model, Watson-Daniels2022}. Notably, some of this work has defined measures and developed methods for evaluating predictive multiplicity in binary classification~\cite{Marx2019} and probabilistic classification~\cite{Watson-Daniels2022, rashomoncapacity}, focusing on so called ``ambiguity'' in models' predictions (i.e., the amount of disagreement in models' predictions on different points). Our work is the first to extend the analysis of multiplicity to the problem of predictive allocation under resource constraints.   We introduce measures of multiplicity for both single-target and multi-target settings, and introduce efficient methods that, for a subset of points, can certify whether those points contribute to the multiplicity measure. 
\textit{Resource constraints and fairness.}  Recent work on algorithmic fairness has noted the importance of considering resource constraints.  For instance, \citet{black2022algorithmic} discuss how the increased cost of auditing more complex tax filings can lead to prediction-based auditing strategies that disproportionately focus on lower income earners.  Other work has emphasized the importance of considering resource constraints in the context of algorithmic fairness in healthcare~\cite{pfohl2022net} and business analytics~\cite{dearteaga2022algorithmic}.  Our work provides a conceptual and computational framework for reasoning about fairness in the presence of resource constraints.

\vspace{-1em}
\vspace{3pt}
\section{Predictive multiplicity with resource constraints}
\label{sec::multiplicity}
In this section, we introduce a framework for examining single-target predictive multiplicity under resource constraints. Our goal is to study predictive consistency over models with near-optimal performance for each target option. We provide key definitions for predictive multiplicity in~\textsection\ref{sec:sub-ambiguity-def-single-target}, present a computational framework based on mixed-integer programming (MIP) for linear models in~\textsection\ref{sec:sub-compute-amb-single} and introduce methods for improving computational efficiency in~\textsection\ref{sec:improve-efficiency-single-target}. 

\subsection{Preliminaries}

We consider a dataset, $\data{} = \{(x_i, a_i, \yobs{k}_i)\}_{i=1}^{n}$, consisting of $n$ cases, where $x_i = [1,x_{i1},\ldots,x_{id}] \in \X \subseteq \R^{d+1}$ is a feature vector,  $y_i \in \R$ is an outcome of interest (potentially binary), and  $a_i \in A$ is a protected attribute.  We operate within the prediction-based allocation setting where a limited resource is to be allocated to instances in descending order of the predicted value $\hat y_i = \hat y(x_i)$. If case $i$ is selected, it is allocated $r_i$ resources.   Let $\kappa$ denote the resource cap, and let $i_{(j)} = i_{(j)}(\hat y)$ denote the instance with the $j$th largest value of $\hat y_i$  (so that $i_{(1)}$ is the index with the largest predicted value).  Let $\tau_i = \tau_i(\hat y)$ denote the rank of instance $i$ in \textit{descending} order.  

We assume that resources get  allocated to instances $i_{(1)}, \ldots, i_{(J)}$, where $J$ is the largest value such that
$\sum_{j = 1}^J r_{i_{(j)}} \le \kappa$.  The most common prediction-based allocation setting in practice is where there is simply a limit to the number of cases that can be selected (i.e., $r_i = 1 \ \forall i$, in which case $J = \kappa$).  While we restrict our attention to this setting, all metrics and computational methods can be extended to general $r_i \in \mathbb{R}_{>0}$.

\vspace{-0.5em} 
\subsection{Measuring predictive multiplicity under resource constraints}
\label{sec:sub-ambiguity-def-single-target}

Predictive multiplicity is the extent to which models with near-equivalent performance produce different predictions or classifications.  The set of models under consideration is often referred to as the set of ``good models'' \cite{dong2019variable}.

In prior work, \citet{Marx2019} introduced predictive multiplicity metrics for binary classification, and \citet{Watson-Daniels2022} considered the setting of probabilistic classification.
As in the standard predictive multiplicity setting \cite{Marx2019}, we begin with a \emph{baseline model}, $h_0$, that is the solution to an empirical risk minimization (ERM) problem of the form
$ \min_{\clf{} \in \Hset} L({\clf{}; \data{}})$,
over a hypothesis class, $\Hset$, with loss $L(~\cdot~; \data{})$.  In this context, one can consider the $\epsilon$-Rashomon set, which is the set of all models that achieve near-optimal loss. 

\begin{definition}[$\epsilon$-Rashomon set]
For a baseline model $\clf{}_0$ and error tolerance $\epsilon>0$, the $\epsilon$-Rashomon set of competing models is:

\begin{align*}
  \!  \epsset{\clf{}_0} \! :=\!  \{\clf{} \in \Hset: L(\clf{}; \data{}) \leq L(\clf{}_0; \data{})  +  \epsilon\}.
\end{align*}
\end{definition}
\vspace{-0.5em}
In \cite{Marx2019}, $\Hset$ is assumed to be a class of binary classifiers, and one of the predictive multiplicity measures the authors introduce is the \textit{ambiguity} of a prediction problem, 
\[
\alpha_\epsilon(h_0) = \frac{1}{n} \sum_{i = 1}^n \max_{h \in \epsset{\clf{}_0}} \indic{h(x_i) \neq h_0(x_i)}.
\]
Note that under this definition, a prediction problem will have high ambiguity if the positive classification rate, \mbox{$\frac{1}{n}|\{i : h(x_i) = 1\}|$}, differs greatly between $h_0$ and models in $\epsset{\clf{}_0}$.  That is, a high ambiguity may simply result from models that allocate a very different number of resources.

To define an analogous measure for the resource constrained setting, we need to compare models at the same resource cap, $\topK$.  Recall that, unlike in \cite{Marx2019}, we consider $\Hset$ that is a class of prediction models returning continuous values in $\mathbb{R}$, not binary classifiers.   Given a prediction model $h$ and resource cap $\topK$, let
\begin{align}\textit{Top}_{(i,\clf{},\kappa)} = \indic{\tau_i(\clf{}) \leq \kappa},
\end{align} 
be the indicator of whether instance $i$ is ``in the top-$\topK$'' when ranked according to the predicted values, $h$.  
We define two notions of ambiguity in this setting over a dataset sample $S \subset \data{}$.  

\begin{definition}[Top-$\topK$ ambiguity (all)]
\label{Def::Ambiguity-all}
 The {\em $(\epsilon,\topK)$-ambiguity (all)} over a sample, $S$, is the proportion of examples for which the top-$\kappa$  decision changes over the $\epsilon$-Rashomon set: 
\begin{align}
  \ambiguity{\epsilon}{\topK}{\clf{}}{S} 
  &:= 
  \frac{1}{|S|} \sum_{i \in S} \max_{\clf{} \in \epsset{\clf{_0}}}
  \indic{ \textit{Top}_{(i,\clf{},\topK)} \neq \textit{Top}_{(i,\clf_0,\topK)}} .
  \label{Eq::Ambiguity-all} 
\end{align}
\end{definition}

\begin{definition}[Top-$\topK$ Ambiguity (top)]
\label{Def::Ambiguity-top} 
The  {\em $(\epsilon,\topK)$-ambiguity (top)} over a sample, $S$,  is the proportion of top-$\topK$ examples according to $h_0$ that fall outside the top-$\topK$ for some models in the $\epsilon$-Rashomon set: 

\begin{align}
\begin{split}
  \ambiguity{\epsilon}{\topK}{\clf{}}{S}
  &:= 
  \frac{1}{\topK} \sum_{i \in S} \max_{\clf{} \in \epsset{\clf{_0}}} \textit{Top}_{(i,\clf_0,\topK)}
  \left( 1 -  \textit{Top}_{(i,\clf{},\topK)} \right).
  \label{Eq::Ambiguity-top-2} 
  \end{split}
\end{align}
\end{definition}

In addition to ambiguity over a sample, we can think about predictive consistency at the individual level. For an individual, we ask whether there is a model in the $\epsilon$-Rashomon set that can flip the top-$\topK$ selection decision. If there is a near-optimal model that flips the top-$\topK$ decision, then we can say the point is \emph{flippable}.

\begin{definition}[Flippable point]
    An instance $i$ is \textit{flippable} in $\epsset{\clf{_0}}$ if either 
    \begin{align*}
    Top_{(i,\clf{_0},\topK)} = 1 &\text{ and } \max_{h \in \epsset{\clf{_0}}} \tau_i(h) > \topK \quad; \text{or} \\
    Top_{(i,\clf{_0},\topK)} = 0 &\text{ and } \min_{h \in \epsset{\clf{_0}}} \tau_i(h) \le \topK. \label{Eq::flippable}
    \end{align*}
\end{definition}
Note that the top-$\topK$ ambiguity (all) is simply the fraction of instances that are flippable.  Top-$\topK$ ambiguity (top) is the fraction of instance in the top-$\topK$ of the baseline model $h_0$ that are flippable out of the top-$\topK$ by some $h \in \epsset{\clf{_0}}$.  

\subsection{Computing top-$\topK$ ambiguity for linear models}
\label{sec:sub-compute-amb-single}
In this section we describe the procedure for computing the two notions of top-$\topK$ ambiguity for linear models, $\Hset = \{h(x) = x^Tw : w \in \R^{d+1} \}$, and squared error loss, $L(h;\data) = L(w;\data) = RSS(w; \data) =  \sum_{i = 1}^n (y_i - x_i^Tw)^2$.  We use $h$ and $w$ notation interchangeably in the context of linear models.  Unless stated otherwise, we will assume throughout this section that the design matrix $X$ has been transformed to be orthonormal.  The problem is invariant to this operation, but working with an orthonormal X helps simplify expressions and reduce notational burden.

We begin by training the \emph{baseline model} $h_0$ that produces a ranking for each individual in our sample. Our goal is to determine the most meaningful change to this baseline rank for each point over the $\epsilon$-Rashomon set of competing models. Therefore, for instances with a baseline rank in the top-$\topK$, $Top_{(i,\clf{}_0,\topK)} = 1$, we calculate the \textit{maximum} attainable rank for this individual, $\max_{h \in \epsset{\clf{_0}}} \tau_i(h)$. For instances with a baseline rank outside the top-$\topK$, $Top_{(i,\clf{}_0,\topK)} = 0$, we calculate the \textit{minimum} attainable rank for this individual, $\min_{h \in \epsset{\clf{_0}}} \tau_i(h)$. Based on these minimum and maximum ranks, we can compute the proportion of examples whose baseline rank flips over the $\epsilon$-Rashomon set of competing models.

We employ integer programming for this computation. Prior work involves constructing a pool of candidate models that change individual predictions~\cite{Marx2019, Watson-Daniels2022}. From that pool of models, those with near-optimal performance are selected to compute ambiguity. These methods are indirect in that the MIPs do not directly constrain these candidate models to be within the $\epsilon$-Rashomon set. In our setting, we develop a MIP that does include this constraint. The following proposition allows us to neatly characterize the $\epsilon$-Rashomon set, $\epsset{\clf{_0}}$, for linear models. The proof is presented in Appendix~\ref{sec:app-single-target-details}.
\begin{proposition}
    Assume the design matrix $X_{n \times (d+1)}$ has been orthonormalized, and $w_0 = \argmin_{w \in R^{d+1}} \|y - Xw\|_2^2$ is the least squares solution. Then
    \begin{align*}
    \begin{split}
        \epsset{w_0} \! =\!  \{w \in \R^{d + 1}: RSS(w) \leq RSS(w_0)  +  \epsilon\} \\ 
        = \{ w \in \R^{d + 1} : \|w - w_0 \| \le \epsilon \}.
    \end{split}
    \end{align*}
    \label{prop:ortho-eps-level-linear}
\end{proposition}
\vspace{-1em}
With this result in hand, we formulate a MIP to calculate the minimum and maximum rank assigned to each point over $\epsilon$-Rashomon set of competing models, which we call $\verb|FlipTopKMIP|(h_0, x_i; \kappa, \epsilon)$:

\begin{subequations}
\label{IP::Alternate}
\begin{equationarray}{@{}r@{}l@{\,}l@{\,}l>{\,}l>{\,}r@{\;}}
\min \text{ or } \max_{I \in \mathbb{R}^{2 \times (n - 1)}, w \in \mathbb{R}^{d + 1}} & \quad \sum_{\twopt \neq \onept} \predIndic_{(\twopt > i)} & & & & \notag \\ 
\st \notag \\
\predIndic_{(\twopt > i)} + \predIndic_{(i > \twopt)} & = 1  & \forall{\twopt}\in S \setminus i  \label{alt1}
\\ [3pt]
(x_{\twopt} - x_\onept)^T w & \leq  M * \predIndic_{(\twopt > i)} & \forall{\twopt}\in S\setminus i \label{alt2a} \\ [3pt]
(x_{\onept} - x_\twopt)^T w & \leq  M * \predIndic_{(\twopt > i)} & \forall{\twopt}\in S\setminus i \label{alt2b} \\ [3pt]
\|w - w_0\|_2^2 &\le \epsilon & & \label{alt3}\\ [3pt]
w_j & \in  \R & \miprange{j}{1}{d+1}  \label{alt4} \\ [3pt]
\predIndic_{(\twopt > i)}, \predIndic_{(i > \twopt)} & \in  \{0,1\} & \forall{\twopt}\in S\setminus i  \label{alt5} 
\end{equationarray}
\end{subequations}
where we set (see Appendix \ref{sec:app-M-bound-single-target} for details),
\[
M = \left( \sqrt{\|w_0\|_2^2 + \epsilon} \right) \max_{i, j}\|x_j - x_i\|_2 .
\]

The objective minimizes (or maximizes) the rank assigned to an individual $i$. The binary variables $\predIndic_{(\twopt > i)}$ serve as indicators that one point ranks higher than another: $\hat y_\twopt = x_\twopt^T w \ge x_\onept^T w = \hat y_\onept$. So the objective $\sum_{\twopt \neq \onept} \predIndic_{(\twopt > i)} = \tau_i(w) - 1$ is simply the rank (minus $1$) of point $i$ in model $w$. Constraint~\eqref{alt1} says that between two points, one point has to rank higher or lower making sure there are no ties. We connect the indicators to the rank definition through constraints \eqref{alt2a} and \eqref{alt2b} where the rank relationship is established using ``Big-M'' variable $M$. Constraint \eqref{alt3} enforces that $w$ is in the $\epsilon$-Rashomon set, as per Proposition~\eqref{prop:ortho-eps-level-linear}. 

$\verb|FlipTopKMIP|(h_0, x_i; \kappa, \epsilon)$ outputs the minimum or maximum rank assigned to each individual in our sample over the $\epsilon$-Rashomon set. We use this output to determine which points are flippable based on definition 3.4. Then, we simply calculate the proportion of flippable instances to arrive at top-$\kappa$ ambiguity.

\subsection{Improving efficiency by identifying provably (un)flippable points}
\label{sec:improve-efficiency-single-target}

Whereas prior related work on predictive multiplicity in binary \cite{Marx2019} and probabilistic \cite{Watson-Daniels2022} classification has involved solving a MIP for every point in $\data$, we show this is not necessary in our setting.  Specifically, we show that (i) one can efficiently determine that many points are provably \textit{not flippable} over the $\epsilon$-Rashomon set; and (ii) one can identify a subset of \textit{flippable} points by solving a proxy optimization problem with a closed-form solution that produces a $w \in \epsset{w_0}$ that may flip some points into the top-$\topK$.  This means that in practice we only need to solve the computationally expensive \verb|FlipTopKMIP| for a very small subset of points whose flippability remains undetermined following the two efficient filtering steps.  Our approach is grounded in the following three results, whose proofs appear in Appendix \textsection\ref{sec:app-unflippable}.

\begin{proposition}[Prediction gap bound over the $\epsilon$-Rashomon set]
    Define $\Delta_{\onept, \twopt}(w) := \hat y_\onept - \hat y_\twopt =  x_\onept^T  w - x_\twopt^T w$ to be the prediction gap between instances $\twopt$ and $\onept$ under model $w$.  For all $\onept, \twopt$ and $w \in \epsset{w_0}$,
    \[
    \Delta_{\onept,\twopt}(w) \le  \Delta_{\onept,\twopt}(w_0) + \sqrt{\epsilon}\|x_{\onept} - x_\twopt\|_2 
    =: B(\onept, \twopt ; \epsilon)
    \] 
\label{prop:pred-gap-bound}
\end{proposition}
\vspace{-2em}
\begin{corollary}[Provably unflippable points] Suppose $\onept$ is not in the top-$\topK$ for model $w_0$; 
 i.e., $Top_{(\onept,w_0,\topK)} = 0$. If
$\#\{\twopt: B(\onept, \twopt ; \epsilon) < 0\} \ge \topK,$
then $Top_{(\onept,w,\topK)} = 0 \ \ \forall w \in \epsset{w_0}$.
\label{cor:unflippable}
\end{corollary}
Conceptually, Proposition~\ref{prop:pred-gap-bound} establishes a bound on the gap between the predicted values of any two points over the whole $\epsilon$-Rashomon set in terms of the prediction gap under the baseline model, $w_0$.  Corollary~\ref{cor:unflippable} then says that if there are at least $\topK$ points, $\twopt \neq \onept$, whose predicted value is guaranteed to exceed that of point $\onept$ for every model $w \in \epsset{w_0}$, then $\onept$ is unflippable.   

\begin{proposition}[Prediction maximizing model] The predicted value of point $i$, $\hat y_i = x_i^T w$, over the $\epsilon$-Rashomon set is maximized at,
\begin{equation}
w^* = \argmax_{w \in \epsset{w_0}} \hat y_i (w) = w_0 + \sqrt{\epsilon} \frac{x_i}{\|x_i\|_2}.
\end{equation}
\label{prop:max-yhat}
\end{proposition}
\vspace{-1em}
For points that are not ruled out by Corollary~\ref{cor:unflippable}, Proposition~\ref{prop:max-yhat} provides a candidate model within the Rashomon set that may flip a point into the top-$\topK$.  Note that this result does not preclude the possibility that $Top_{(i, w^*, \topK)} = 0$ while also $Top_{(i, w', \topK)} = 1$ for some other $w' \in \epsset{w_0}$.  Taken together, these results often significantly reduce the number of points for which one needs to run the MIP in order to determine their flippability.  





\section{Multi-target Multiplicity and Fairness}
\label{sec:multiple-target-multiplicity}
In the previous section, we introduced the top-$\topK$ ambiguity measure for characterizing predictive multiplicity for a single target, $y$, over the $\epsilon$-Rashomon set.  As discussed at the outset, an important potential source of multiplicity is in the specification of the target outcome itself.  In this section, we introduce a measure of multi-target multiplicity for the setting where the multiple targets will ultimately be combined in some way to produce a single score that will be used to prioritize allocation.  We also discuss group fairness by examining how the selection rate for a given group varies depending on the specific choice of combining rule.

\subsection{Multi-target ambiguity and index models}
\label{sec:multiple-target-ambiguity}
Given candidate targets, $\yobs{1}, \ldots, \yobs{K}$, and features $X$, we consider a family of ``combining procedures,'' $\combine_\alpha$, parameterized by $\alpha$ that map from training data $(X,\yobs{1}, \ldots, \yobs{K})$ to the space of prediction models $\alphaset{} = \{\clf_\alpha: \mathcal{X} \mapsto \mathbb{R}\}.$  Under a resource constraint of $\topK$, resources will then be allocated to the units with the $\topK$ highest values of $\clf_\alpha (x_i)$.  We are interested in characterizing how the top-$\kappa$ set varies across the parameter $\alpha$ governing the combining procedure, $\combine_\alpha$.  More formally, we define \textit{multi-target ambiguity} as follows.
\begin{definition}[Multi-target ambiguity]
\label{Def::Ambiguity-alpha}
 The {\em $(\alpha, \topK)$-multi-target--ambiguity} of a combining procedure $c_\alpha$ over a sample $S$ is the proportion of examples whose top-$\topK$ decision varies depending on the choice of $\alpha$, 
\begin{align}
\begin{split}
  A_{\alpha,\topK}(S)
  &:= 
  \frac{1}{|S|} \sum_{i \in S} \max_{\clf_\alpha, \clf_{\alpha'} \in \alphaset{}}
  \indic{ \textit{Top}_{(i,\clf_{\alpha},\topK)} \neq \textit{Top}_{(i,\clf_{\alpha'},\topK)} }. 
  \label{Eq::Ambiguity-alpha} 
  \end{split}
\end{align}
\end{definition}
\noindent Whereas in the single target case 
we were interested in ambiguity over the $\epsilon$-Rashomon set, here we focus on ambiguity over the \textit{combining procedure}. Conceptually, a point is ``ambiguous'' if whether it is in the top-$\topK$ depends on the particular choice of $\alpha$ in the combining procedure.  The family of models generated by the combining parameters $\alpha$ is the multi-target family of ``good models.''

To make the discussion more concrete, we introduce two combining procedures inspired by existing practice, the \textit{index model} approach and the \textit{index variable} approach.  

\begin{definition}[Index model]  The \textit{index model} is defined as 
\begin{equation}
\hat y_{IM}(x; \alpha) = \combine^{IM}_\alpha(X,\yobs{1}, \ldots, \yobs{K})(x) =  \sum_{k = 1}^K \alpha_k \yhat{k}(x),
\label{eq:index-model}
\end{equation}
where $\alpha$ is a weight vector in the $K$-simplex, $\alpha \in \Simplex{K} := \{\alpha \in \R^K : \sum_{k = 1}^K \alpha = 1, \alpha_k \ge 0 \ \forall k\}$, and $\yhat{k}(x)$ is a prediction model for target $\yobs{k}$.       
\end{definition}  

Note that for this definition to make sense, we assume that the individual predictors $\yhat{k}$ are first standardized to an appropriate common scale, such as by rescaling $\hat y \gets \frac{\hat y - mean(\hat y)}{sd(\hat y)} $ or converting to percentiles prior to combining.  The choice of standardization function does influence results.    Choosing a single target outcome $k_0$ is a special case of an index model with $\alpha_{k_0} = 1 $ and $\alpha_k = 0$ for $k \neq k_0$.  
An advantage of the index model approach is that it places no restrictions on the training procedure used to construct $\hat y^{(k)}$.  Where appropriate, multi-task learning approaches can be used to jointly learn models across the targets.

This approach is motivated by existing practice in domains such as criminal justice and human services, where multiple so-called scales (i.e., $\yhat{k}$'s) are constructed to predict different outcomes, and are then aggregated into prioritization schemes or decision recommendations.  For instance, the Allegheny Housing Assessment (AHA) tool used to prioritize housing services for persons experiencing homelessness sums the predictions of three $\yobs{k}$ assessed within 12 months of the assessment date: (i) the likelihoood of inpatient mental health services; (ii) the likelihood of jail booking; and (iii) the likelihood of 4 or more ER visits \cite{kithulgoda2022predictive}.%
\footnote{These tools are presented as examples of models that have been constructed for real world applications. We are not endorsing the use of these other tools.}  

An alternative to index models is an index variable approach, where instead of first forming predictions and then aggregating the different scales, a composite target outcome is formed and then that target is predicted.  
\begin{definition}[Index variable]
    Given candidate targets $\yobs{1}, \ldots, \yobs{K}$, features $X$, and weights $\alpha \in \Simplex{K}$, an \textit{index variable} model, $\hat y_{IV}(x; \alpha)$ is defined by the minimizer,
\begin{equation}
\hat y_{IV} (x; \alpha) = \min_{\clf{} \in \Hset} L(\clf{}; \yobs{\alpha}), \quad \text{ where } \quad \yobs{\alpha} = \sum_{k = 1}^K \alpha_k \yobs{k}.
\label{eq:index-variable}
\end{equation}
\end{definition}
  Conceptually, the index variable approach can be thought of as first forming a composite outcome that is believed to more comprehensively describe some latent quantity, and then finding the optimal predictor for that outcome.  
  Both for the index model and index variable formulation, the parameter $\alpha$ captures potential underspecification in the choice of target.    In the case of linear models, the index model and index variable approach coincide. 

\begin{proposition}[Equivalence of index model and index variable approaches for linear models.]
    If we restrict consideration to linear models whose solution takes the form $\hat y = M_X y$ for some $n \times n$ matrix $M_X$ that depends on $X$ but not on $y$, then the index model and index variable approach are equivalent.
    \label{prop:equiv-im-iv-linear}
\end{proposition}
\noindent A proof is provided in Appendix \textsection\ref{sec:app-im_iv_equiv_proof}.  Note that linear regression is a special case of a linear model, with $M_X = X(X^TX)^{-1}X^T$.  Other models such as regression splines fall into this class as well.  

In the remainder of this work we focus on the index model approach, as it can be analysed in a computationally tractable way for general predictors $\yhat{k}$.  Due to the equivalence result, our methods are directly applicable to the index variable approach in the case of linear models.

\subsection{Computing multi-target top-$\kappa$ ambiguity for index models}
\label{sec:compute-ambiguity-multi-target}

In this section, we introduce a MIP for computing multi-target ambiguity as defined in Eq.~\eqref{Eq::Ambiguity-alpha} for the family of index models.  The MIP calculates the minimum and maximum rank attainable for each individual point over the combining parameters, $\alpha$.  The multi-target ambiguity is then given as the proportion of points for which the minimum rank is $\le \kappa$ while the maximum rank $\ge \kappa$.  %

For combining procedures parameterized by $\alpha$, the min and max rank of each individual $i \in S$ can be obtained by solving the optimization problem, which we call $\verb|FlipTopKMultiMIP|(x_i; \kappa)$: 

\begin{subequations}
\label{IP::IndexModel}
\begin{equationarray}{@{}r@{}l@{\,}l@{\,}l>{\,}l>{\,}l@{\;}}
\min_{I \in \{0,1\}^{n - 1},\, \alpha \in \mathbb{R}^K} & \sum_{\twopt \neq i} \predIndic_{(\twopt > i)} - 0.5 \sum_{k = 1}^K \alpha_K  & \text{ or } & & & \notag   \\
\max_{I \in \{0,1\}^{n - 1},\, \alpha \in \mathbb{R}^K} & \sum_{\twopt \neq i} \predIndic_{(\twopt > i)} + 0.5 \sum_{k = 1}^K \alpha_K & & & & \notag \\ 
\st \notag \\ [2pt]
\predIndic_{(\twopt > i)} + \predIndic_{(i > \twopt)} & =  1 & \miprange{\twopt}{1}{n}\setminus i \label{im1} \\ [2pt]
\hat y_{IM}(x_\twopt; \alpha) - \hat y_{IM}(x_i; \alpha) & \leq  M * \predIndic_{(\twopt > i)} & \miprange{\twopt}{1}{n} \setminus i \label{im2a} \\ [2pt]
\hat y_{IM}(x_i; \alpha) - \hat y_{IM}(x_\twopt ; \alpha) & \leq  M * \predIndic_{(i > \twopt)} & \miprange{\twopt}{1}{n} \setminus i \label{im2b} \\ [2pt]
0 \le \alpha_k &\le  1 &  \miprange{k}{1}{K} \label{im3}\\ [2pt]
0.1 \le  \sum_{k = 1}^K \alpha_k &\le 1 & \label{im4}\\ [2pt]
\alpha_k & \in  \R & \miprange{K}{0}{d}  \label{im5} \\ [2pt]
\predIndic_{(\twopt > i)}, \predIndic_{(i > \twopt)} & \in  \{0,1\} & \miprange{\twopt}{1}{n}\setminus i  \label{im6}
\end{equationarray}
\end{subequations}
where $\hat y_{IM}(x_i; \alpha)$ is shorthand for $\sum_{k = 1}^K \alpha_k \yhat{k}(x_i)$, and
\begin{equation*}
M = \max_{\twopt,k} \hat y^{(k)} (x_\twopt) - \min_{i,k} \hat y^{(k)}(x_i).
\end{equation*}

$\verb|FlipTopKMultiMIP|(x_i; \kappa)$ fits the parameters $\alpha$ that minimize (or maximize) the rank assigned to individual $i$. The objective minimizes (or maximizes) the sum of individuals ranked higher than individual $i$ with an additional term in the objective that forces $\sum \alpha_k = 1$. Again, the binary variables $\predIndic_{(\twopt > i)}$ serve as indicators that one point ranks higher than another. And constraint~\eqref{im1} says one point has to rank higher or lower than another (i.e. there are no ties). We connect the indicators to the ordering relations $\hat y_{IM}(x_\twopt; \alpha) \ge  \hat y_{IM}(x_\onept; \alpha)$ and $\hat y_{IM}(x_\twopt; \alpha) < \hat y_{IM}(x_\onept; \alpha)$ through constraints \eqref{im2a} and \eqref{im2b}, introducing the ``Big-M'' variable, $M$. Constraints \eqref{im3} and \eqref{im4} ensure are a soft version of the constraint that $\alpha$ is in the simplex, $\Simplex{K}$. %

This formulation is similar to \verb|FlipTopKMIP| from the single-target case, but the objective has an additional term, and the optimization here is over the combining weights $\alpha$ rather than the parameters of the individual predictors $\yhat{k}$. As in the single target setting, we calculate ambiguity by identifying flippable points using a MIP.  In this setting, there is no baseline model, so the term "flippable" now refers to points where there exist two choices of combining parameters, $\alpha \neq \alpha'$, such that $Top_{(i, \alpha, \kappa)} \neq Top_{(i, \alpha', \kappa)}$.  Furthermore, the optimization here is no longer over an $\epsilon$-Rashomon set---a notion which does not naturally extend to the multiple target setting due to the absence of a baseline model---but rather over the parameters $\alpha$ governing the combining rule.  

As in the single-target context, we can once more reduce the number of times we need to run the MIP by identifying points that provably cannot appear in the top-$\topK$ set for any choice of $\alpha$, and characterize the prediction-maximizing choice of $\alpha$ for each point.  The results and accompanying proofs are in Appendix \textsection\ref{sec:app-unflippable-IM}. 

\subsection{Group-level selection rates in top-$\topK$ selection with multiple targets }
\label{sec::multi_fair}

In this section we demonstrate how our framework can be applied to examine group fairness concerns. 
 Specifically, we consider how the selection rate---i.e., the proportion of instances from a given group, $A = a$, in the top-$\topK$---varies with the combining weights $\alpha$.\footnote{While, in principle, one can also examine measures such as the False Positive Rate and True Positive Rate by analysing the subsample of instances for which $\yobs{k} = 0$ (or 1, for TPR), it is not entirely clear how such quantities should be interpreted.  How should one weigh a high FPR for a given target against a low FPR for a different one in a setting where the "correct" choice of target is itself in doubt?}

We can compute the combining parameters that maximize the number of individuals in a given group who are selected to be in the top-$\topK$. That is, we consider,
\begin{equation}
\min_\alpha or \max_{\alpha} \sum_{i = 1}^n \indic{A = a} Top_{(i, \alpha, \topK)} = \min_\alpha or \max_{\alpha} \sum_{i \in G_a}^n  Top_{(i, \alpha, \topK)} \label{eq::multiFAIR}
\end{equation}
where $G_a = \{i : A_i = a\}$ denotes all the instances that are in protected group $A = a$.
While the goal of our work is to characterize the variation in selection rates afforded by different choices of combining parameters, $\alpha$, the methods can also be used to select a particular model that maximizes (or minimizes) those rates.

We compute the quantities in Eq.~\eqref{eq::multiFAIR} through another MIP.  For this purpose we introduce variables $T_i \in \{0,1\}$ that play the role of the $Top_{(i, \alpha, \topK)}$ indicator. We refer to this MIP as \\ \verb|GroupSelectRateTopKMultiMIP|$(a;\kappa)$: 

\begin{subequations}
\label{IP::IndexModelFairness-reduced}
\begin{equationarray}{@{}c@{}l@{\,}l@{\,}l>{\,}l>{\,}r@{\;}}
\min_{I \in \{0,1\}^{2n \times |G_a|},\ T \in \{0,1\}^{|G_a|}, \ \alpha \in \Simplex{K}} & \sum_{i \in G_a} T_i - 0.5 \sum_{k = 1}^K \alpha_k  & \text{ or } & & &  \notag   \\
\max_{I \in \{0,1\}^{2n \times |G_a|},\ T \in \{0,1\}^{|G_a|}, \ \alpha \in \Simplex{K}} 
& \sum_{i \in G_a} T_i + 0.5\sum_{k = 1}^K \alpha_k \notag \\ [2pt]
\st \notag \\
\predIndic_{(\twopt > i)} + \predIndic_{(i > \twopt)} =  1  & \forall{\onept} \in G_a,  \forall{\twopt}\in S \setminus i \label{imf1} \\ [3pt] 
\hat y_{IM}(x_\twopt; \alpha) - \hat y_{IM}(x_\onept; \alpha) \leq  M_I * \predIndic_{(\twopt > \onept)} \ \ \  & \forall{\onept} \in G_a,  \forall{\twopt}\in S\setminus i \label{imf2a} \\ [3pt]
\hat y_{IM}(x_\onept; \alpha) - \hat y_{IM}(x_\twopt; \alpha) \leq  M_I * \predIndic_{(\onept > \twopt)} & \forall{\onept} \in G_a,  \forall{\twopt}\in S\setminus i \label{imf2b} \\ [3pt]
\kappa - \sum_{\twopt \neq \onept} \predIndic_{(\twopt > \onept)} \leq  \kappa * T_\onept & i \in G_a \label{imf3a} \\ [3pt]
\left(1 + \sum_{\twopt \neq \onept} \predIndic_{(\twopt > \onept)} \right) - \kappa   \leq  (n - \kappa) (1 - T_\onept) & i \in G_a \label{imf3b} \\ [3pt]
0 \le \alpha_k \le  1 &  \miprange{k}{1}{K} \label{imf4}\\ [2pt]
0.1 \le \sum_{k = 1}^K \alpha_k \le 1 & \label{imf5}\\ [2pt]
\predIndic_{(\twopt > \onept)}, \predIndic_{(i > \twopt)}  \in  \{0,1\} & \miprange{\onept \neq \twopt}{1}{n}  \label{imf6} \\ [2pt]
T_\onept \in \{0,1\} & i \in G_a \label{imf7}
\end{equationarray}
\end{subequations}

The objective minimizes (or maximizes) the number of individuals in group $G_a$ that are selected to be in the top-$\topK$. There is an additional term in the objective that forces $\sum \alpha_k = 1$, which has the effect of enforcing the simplex constraint on $\alpha$. Recall, the binary variables $\predIndic_{(\twopt > i)}$ serve as indicators that one point ranks higher than another. Thus, constraint~\eqref{imf1} means that one point has to rank higher or lower making sure there are no ties. We connect the indicators to the ordering relations $\hat y_{IM}(x_\twopt; \alpha) \ge \hat y_{IM}(x_\onept; \alpha)$ or $\hat y_{IM}(x_\twopt; \alpha) < \hat y_{IM}(x_\onept; \alpha)$ through constraints \eqref{imf2a} and \eqref{imf2b} using the ``Big-M'' variable $M_I$. This value is set to be the max possible difference in prediction between two points $M_I = \max_{\onept,k} \hat y^{(k)} (x_\onept) - \min_{\onept,k} \hat y^{(k)}(x_i)$. To make sure $T_i$ reflects whether individuals are in group $G_a$ and ranked in the top-$\topK$, we have constraints \eqref{imf3a} and \eqref{imf3b}. Constraints \eqref{imf4} and \eqref{imf5} are other pieces of the simplex constraint. 

\section{Stable points}
\label{sec:stable-points}
Thus far our focus has been on multiplicity and identifying \textit{flippable} points---those whose decision depends on the particular model chosen among the set $\mathcal{G}$ of good models.\footnote{E.g., $\mathcal{G} = \epsset{h_0}$ in the single target setting, or the set of index models parameterized by $\alpha \in \Simplex{K}$ in the multi-target setting. }
In practice, however, we may be equally interested in \textit{un}flippable points.  As prior work has pointed out, the presence of multiplicity raises concerns about arbitrariness: What justification can you offer someone who receives an adverse decision from the chosen model when there may exist another good model that would have given them a favorable decision ~\cite{black2022model}?   Our work can speak to this as well. Concretely, our proposed methods can be used to identify what we call \textit{stable points}: cases whose decisions do not change over the set of good models.   
\begin{definition}
Let $\mathcal{G}$ be the set of ``good models''.  We say that $i_0$ is a \textit{stable $\topK$-selected} point if $Top_{(i_0, h, \topK)} = 1 \,\, \forall h \in \mathcal{G}$. Similarly, we say that $i_0$ is a \textit{stable $\topK$-unselected} point if $Top_{(i_0, h, \topK)} = 0 \,\, \forall h \in \mathcal{G}$. 
\end{definition}
Stable points are instances for which the decision is non-arbitrary: their decisions are invariant to the specific choice of model among those considered acceptable, which is strong justification for the given decision. 
The fraction of stable $\topK$-selected points out of $\topK$ is a useful quantifier of the arbitrariness of a predictive allocation task.  For instance, if this fraction is very low, this may highlight a need for further principled deliberation on the specific choice of model, or affect the willingness to adopt a predictive model for the given allocation task.

\begin{figure*}[th]
    \centering
    \begin{subfigure}{.73\textwidth}
        \centering
        \includegraphics[width=\textwidth]{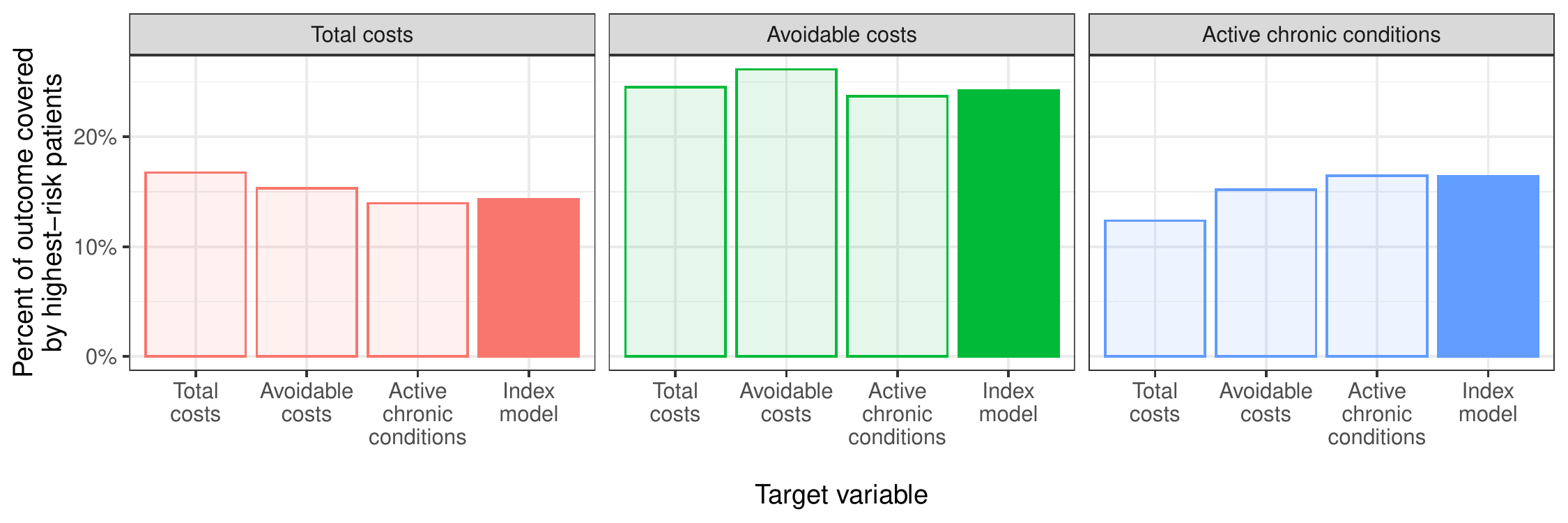}
        \caption{}
        \label{healthcare_concentration_table_as_bar_plot_left}
    \end{subfigure}
    \begin{subfigure}{.24\textwidth}
        \centering
        \includegraphics[width=\textwidth]{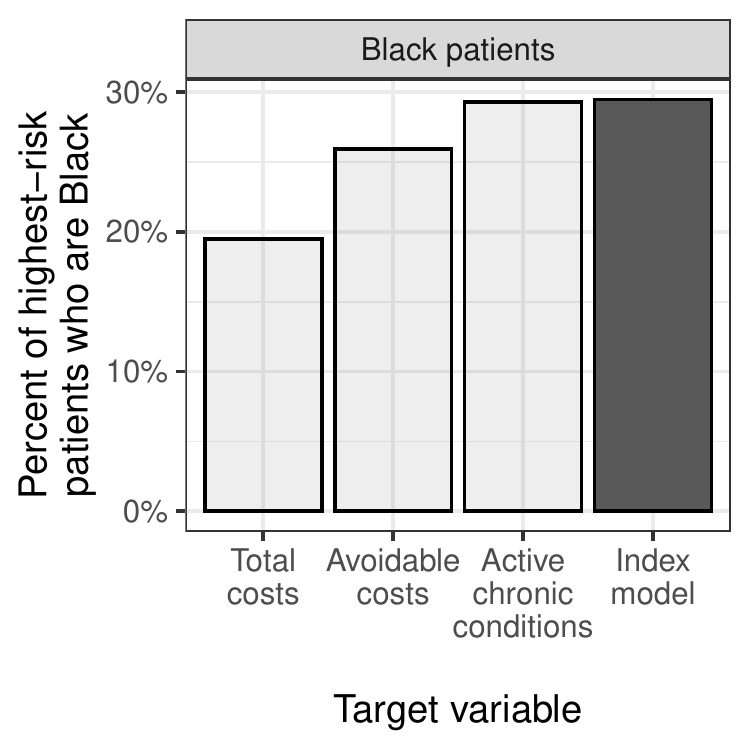}
        \caption{}
        \label{healthcare_concentration_table_as_bar_plot_right}
    \end{subfigure}
    \vspace{-10pt}
    \caption{(A) The concentration of various outcomes under models optimized for different targets. Each panel shows the percent of an outcome captured by the highest-risk patients relative to the entire outcome distribution across all patients. Each bar represents one type of model. The transparent bars depict models trained to predict individual targets, whereas the solid bars depict the index model, which re-weights the individual predictions to maximize fairness. (B) The percent of Black patients among highest-risk patients identified by each model.}
    \label{fig:healthcare_concentration_table_as_bar_plot}
\end{figure*}

\section{Evaluation}
\label{sec:eval}
In this section, we apply the techniques developed above to the healthcare dataset analyzed by Obermeyer et al. to better understand the opportunities afforded by multiplicity among multiple target variables.
First, we describe the dataset in more detail and then apply our multi-target multiplicity framework to it. We then construct a semi-synthetic version of this dataset to develop intuition for the conditions under which we should (or should not) expect to see gains from index models. Finally, we compare the degree of multiplicity that arises in resource constrained settings under a single target to the same under multiple targets.
Throughout this section, we solve all integer programs with Gurobi v.9.5.2~\citep{gurobi}. Our software implementation is at \url{https://github.com/JWatsonDaniels/multitarget-multiplicity}.

\subsection{Dataset}
We demonstrate our framework on a dataset released by Obermeyer et al., which is unique in several ways.
The original paper examines patient data for all primary care patients at a large academic hospital.
However, due to the sensitivity of the data, the authors were unable to release the dataset in its original form.
Instead, they created a publicly available semi-synthetic version of the dataset that is designed to closely mirror the original dataset.\footnote{\url{https://gitlab.com/labsysmed/dissecting-bias}}\looseness=-1

The released dataset contains several related but different outcomes for patients in a given year including total healthcare costs, avoidable healthcare costs (emergency visits and hospitalizations), and number of active chronic illnesses.
It also contains a rich set of features about each patient, including demographics (age, sex, race) and information about the patient's health and healthcare costs in the previous year.
Specifically, there are indicators for individual chronic illnesses that a patient had in the previous year, costs claimed by the patients' insurer in the previous year, biomarkers for medical tests from the previous year, and medications taken in the previous year.

In the original paper, Obermeyer et al. examine a proprietary scoring system used by the hospital to identify high-risk patients.
The risk scores are generated by a model designed to predict healthcare costs in the current year based on patient demographics and healthcare information available from the previous year.
In particular, patients who are assigned risk scores that fall in the 97th percentile or above (i.e., the top 3\% of assigned scores) are automatically identified for inclusion in the hospital's ``high-risk care management'' program.

The authors examine the assigned risk scores in detail and show that they contain a significant racial bias.
Specifically, they find that Black patients at a given risk score have worse health outcomes, on average, than their White counterparts.
The authors trace this bias back to the choice of predicting healthcare costs as the target variable.
Due to differences in access to healthcare, White patients tend to have higher healthcare costs, on average, than Black patients of similar health.
This difference is then reflected in the developed risk score, leading to the observed racial bias.
Obermeyer et al. then go on to show that there are different target variable choices that exhibit less of a racial bias---specifically using either avoidable costs or active chronic illnesses as a target instead of total costs.

\begin{figure*}[!t]
     \centering
     \begin{subfigure}{.25\textwidth}
         \centering
         \includegraphics[width=\textwidth, height=0.2\textheight]{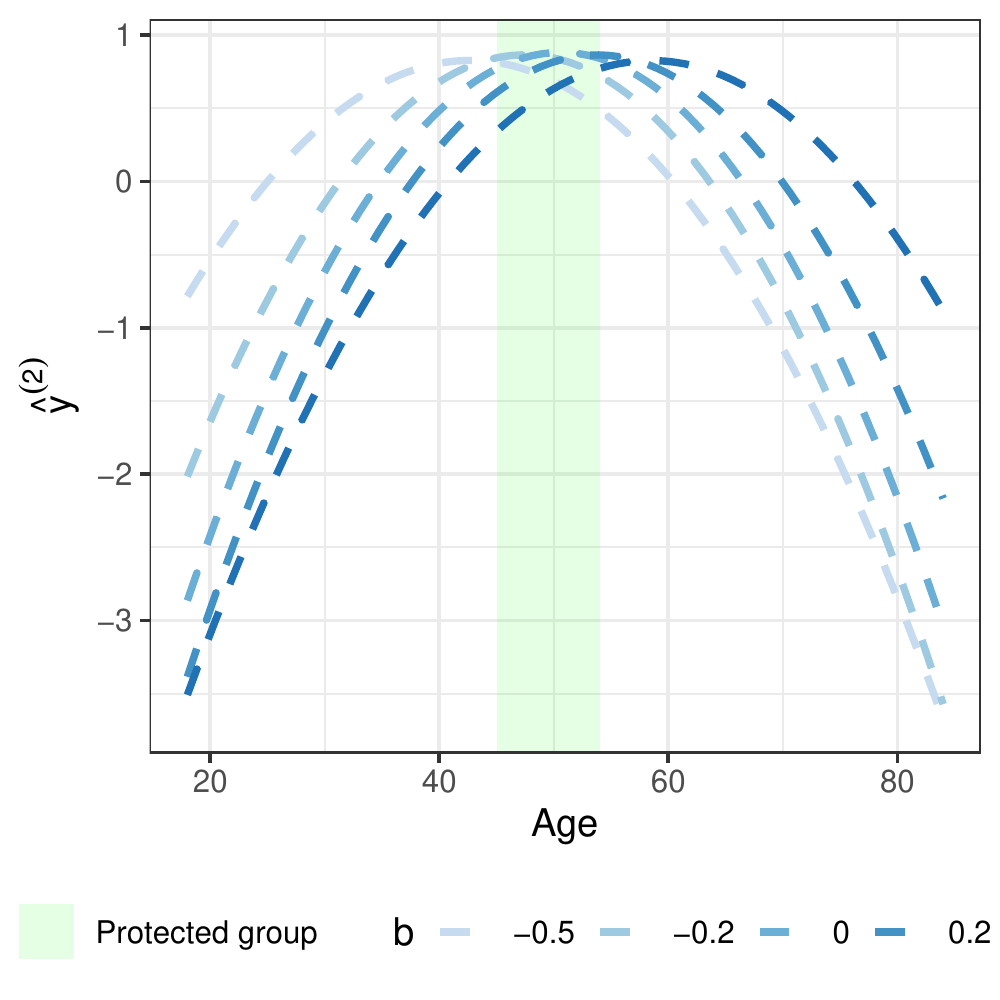}
         \caption{}
         \label{fig:synthetic_data_schematic}
     \end{subfigure}
     \begin{subfigure}{0.45\textwidth}
         \centering
         \includegraphics[width=\textwidth,height=0.2\textheight]{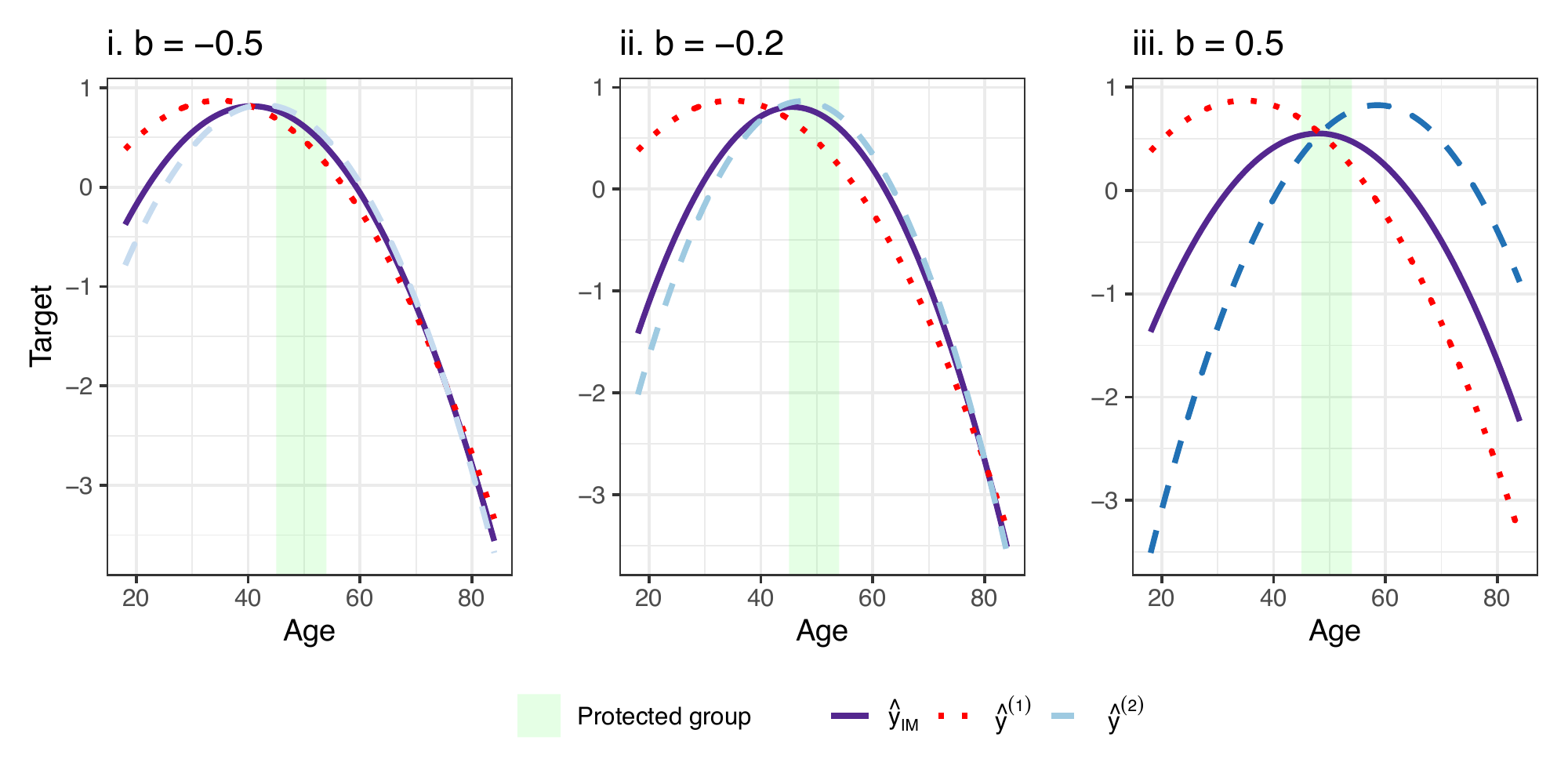}
         \caption{}
         \label{fig:synthetic_data_group_selection_rateB}
     \end{subfigure}
     \begin{subfigure}{0.25\textwidth}
         \centering
         \includegraphics[width=\textwidth,height=0.2\textheight]{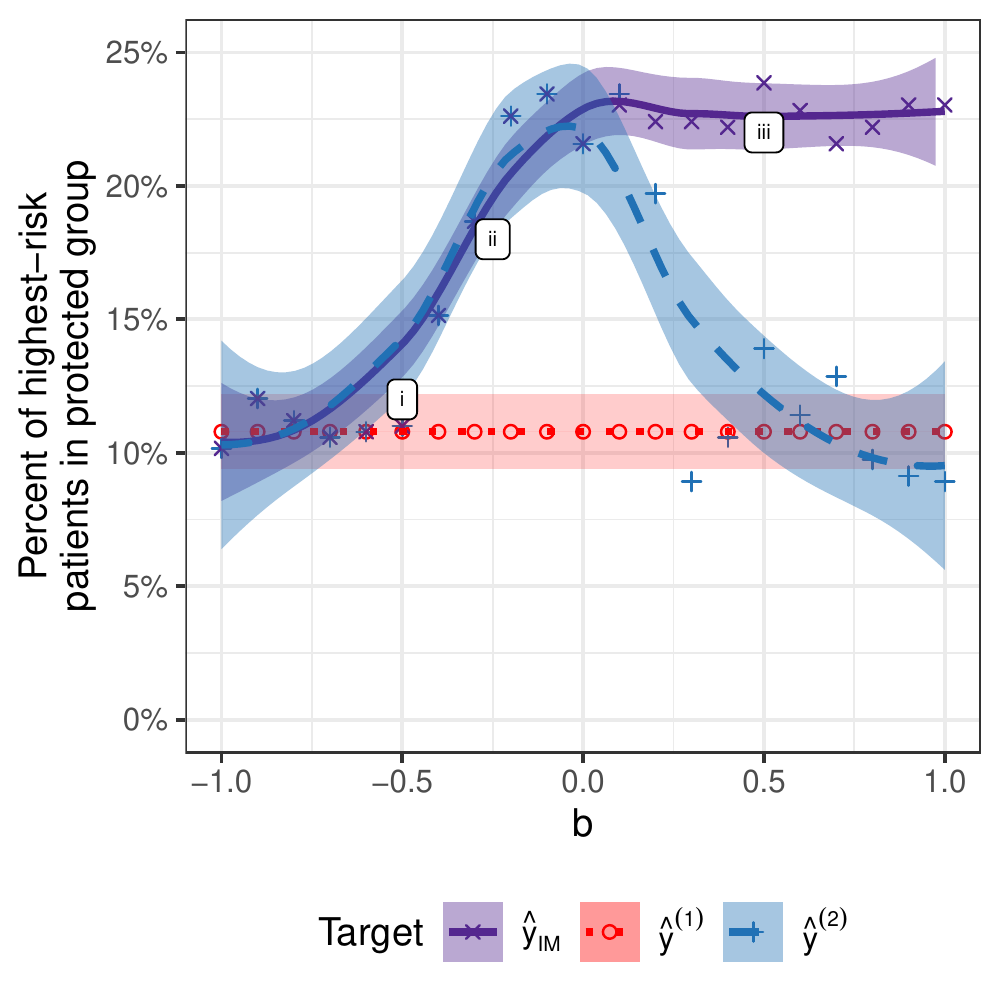}
         \caption{}
         \label{fig:synthetic_data_group_selection_rateC}
     \end{subfigure}
        \vspace{-10pt}
        \caption{(A) A semi-synthetic family of models $\yhat{2}$, which go from negatively correlated with age $(b < 0)$ to positively correlated with age $(b > 0)$. Patients in the protected group are concentrated in the middle age range, indicated by the green band.
        (B) A closer look at all three models for three different values of $b$. In the left panel, none of the models peak in the middle age range. In the middle, the index model $\hat{y}_{IM}$ (solid purple) learns to ignore $\yhat{1}$ (dotted red) in favor of $\yhat{2}$ (dashed blue) to capture more of the protected group. On the right, neither of the individual models peak to capture the protected group, but the index model averages them to do so.
        (C) A more detailed look at the concentration of the protected group found by each model over the range of $b$ values, showing that the index model dominates either individual model over the entire range. %
        }
        \label{fig:semisynthetic_index_model}
\end{figure*}

\subsection{Optimizing across healthcare outcomes}
We present a re-examination of this healthcare dataset to further explore the ways in which flexibility in target variable choice can be used to address fairness concerns.
Obermeyer et al. consider using one of each of the three different target variables, which in our framework corresponds to an index model with binary $\alpha$ weights.
For example, the cost model can be thought of as $\hat y_{IM} = 1 \cdot \yhat{\mathrm{costs}} + 0 \cdot \yhat{\mathrm{avoidable~costs}} + 0 \cdot \yhat{\mathrm{active~illnesses}}$.
However, these are just three extremes among the possible set of index models that can be formed with a continuous $\alpha$ to create a weighted average of the three available target variables.

Our analysis explores whether exercising these extra degrees of freedom can lead to more equitable outcomes.
To address this, we replicate and extend the analysis in Table 2 of the original paper, using the released dataset.\footnote{The original table is generated using the proprietary, unreleased data. Replicating the table with the released dataset produces similar, but not identical results for this reason.}
Specifically, we train separate models to predict each of the three target variables (healthcare costs, avoidable costs, and active chronic illnesses) and use the fitted models to rank a held-out set of patients.\footnote{We use the train/holdout set specified by the authors in the released dataset. We train OLS linear regression models for each variable. In order to do so, we remove several co-linear features provided in the released dataset, detailed in Appendix~\ref{sec:dataset_details}.}

We identify the top 3\% of highest-risk patients according to each of the models and look at the concentration of outcomes and the racial composition of the identified patients.
For instance, when considering total costs, we compute what percent of all costs (across all patients) are covered by just the highest-risk patients.
When considering active chronic illnesses, we instead compute the fraction of all illnesses (across all patients) covered by this set.
We then extend these results by running the multi-target fairness mixed-integer programming (MIP), $\verb|GroupSelectRateTopKMultiMIP|$, to search for an index model that maximizes the fraction of Black patients concentrated among the highest-risk set.
Appendix ~\ref{sec:index_model_diagram} outlines more details of this process.

The results are displayed in Fig.~\ref{fig:healthcare_concentration_table_as_bar_plot} and show several key observations.
First, we see that, as expected, the model trained to predict a given individual outcome has the largest concentration of that outcome in the high-risk patient set.
Notice that the transparent bars on the far left panel of Fig.~\ref{fig:healthcare_concentration_table_as_bar_plot}A show that the model trained to predict total cost is the one that has the highest concentration of total costs in the high-risk patient set.
Conversely, on the far right panel we see that modeling active chronic conditions produces the highest concentration of current illnesses in the high-risk set.
Second, despite these differences we see comparatively small variation in outcome concentration across different target variable choices, with less than a 5 percentage point difference across models in the first three panels.
But, we do see a substantial difference in the racial composition of the high-risk set, as indicated in Fig.~\ref{fig:healthcare_concentration_table_as_bar_plot}B---a more than 10 percentage point difference.

The index model, in comparison, is shown in the solid bars of Fig.~\ref{fig:healthcare_concentration_table_as_bar_plot} for $\alpha = \left( 0.05, 0.0, 0.95\right)$.
By comparing the solid bars to the transparent ones, we see that the index model does a reasonable job of capturing each of the individual targets that it is comprised of, but also produces a high-risk set with a high concentration of Black patients, as per the objective of the multi-target group selection formulation \eqref{eq::multiFAIR}.
In effect, this represents a ``best of both worlds'' solution: we are able to fit separate models that are useful for predicting the three outcomes that may be of interest on their own (i.e., for budgeting purposes), but we also arrive at a way of ranking patients that results in a more equitable allocation of a scarce resource via the index model.

\subsection{Exploring the conditions for effective multi-target optimization}
\label{sec:index_model_intuition}
In the example above, we found it was possible to learn an index model that combined individual target variables from the healthcare dataset to improve group selection rates.
In this section, we use semi-synthetic data to gain a better understanding of the conditions for which we might (or might not) expect to see such gains in other datasets.
To do this, we modify the healthcare dataset to systematically control the relationship between a protected group attribute $a$, a feature $x$, and the different choices of target $\yobs{k}$.
We then vary these relationships and examine how this affects the group selection rate that an index model can achieve.

Specifically, we construct a dataset with one feature (age) and two target variables $\yobs{1}$ and $\yobs{2}$, along with a protected attribute.\footnote{While age is considered a protected attribute under various discrimination laws, for the purposes of our evaluation in a healthcare setting, we treat it as an unprotected attribute.}
We construct the protected attribute to be non-monotonically correlated with age, with a higher concentration of patients in the protected group falling in the middle age range compared to the rest of the population.
We then construct one target variable $\yobs{1}$ that is negatively correlated with age and one target variable $\yobs{2}$ whose correlation with age varies from strongly positive to strongly negative, controlled by a parameter $b$, as shown in Fig.~\ref{fig:semisynthetic_index_model}A.
In this setting, prioritizing middle-aged patients maximizes the fraction of high-risk set that is in the protected group, but fitting a model to $\yobs{1}$ prioritizes young patients, resulting in a lower selection rate.
Conversely, when $b$ is large and positive, $\yobs{2}$ is positively correlated with age, and so fitting a model to it will prioritize older patients, also leading to sub-optimal group selection.
However, as we show in Figs.~\ref{fig:semisynthetic_index_model}B and \ref{fig:semisynthetic_index_model}C, an index model can be fit over a wide range of $b$ values such that the group selection rate is maximized.
The intuition is that the index model can learn to average out unhelpful correlation structure between the protected attribute and the target variables.

\begin{figure*}[!t]
     \centering
     \begin{subfigure}[b]{0.4\textwidth}
         \centering
         \includegraphics[width=\textwidth]{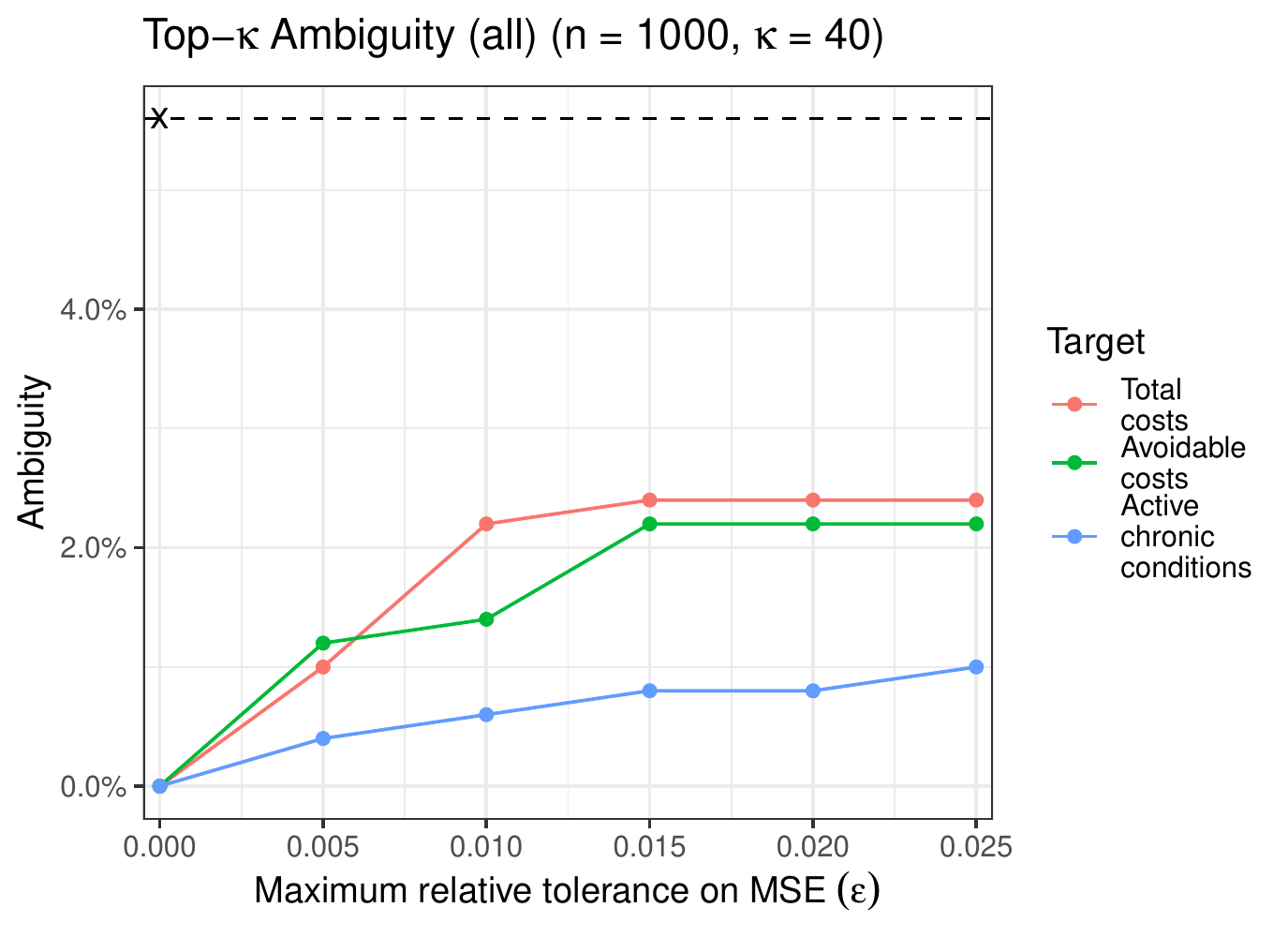}
         \caption{}
         \label{fig:healthcare_single_vs_muli_target_ambiguity}
     \end{subfigure}\hspace{2em}
     \begin{subfigure}[b]{0.4\textwidth}
         \centering
         \includegraphics[width=\textwidth]{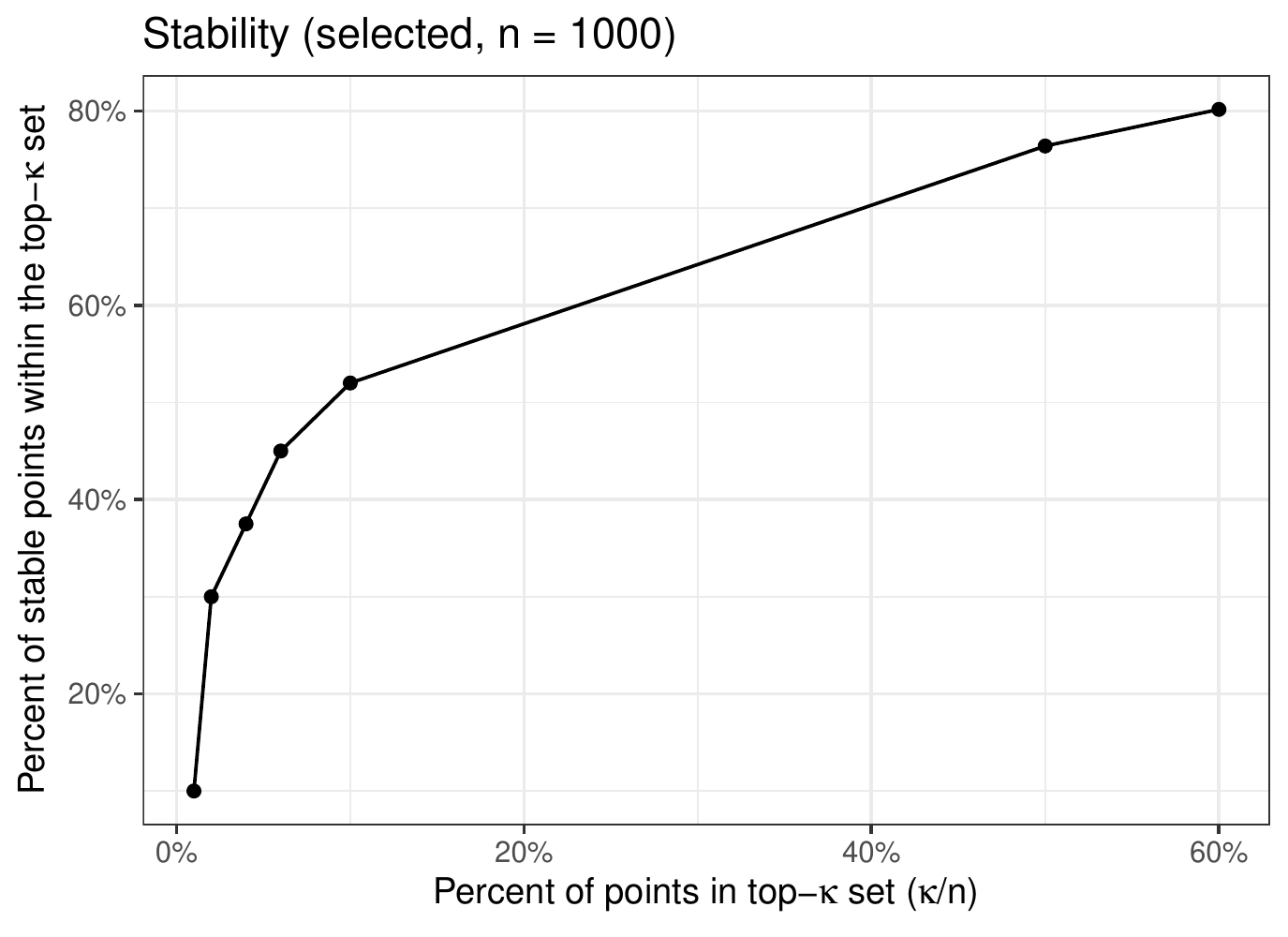}
         \caption{}
         \label{fig:healthcare_multi_target_stability}
     \end{subfigure}
        \vspace{-10pt}
        \caption{(A) Comparison of multiplicity within vs. between targets. Ambiguity within each individual target is shown by the colored lines at different relative mean squared error tolerances $(\epsilon)$. Ambiguity across the three targets, shown by the black 'x' and dotted line, is much higher than ambiguity for any individual target. (B) Stability of points within the top-$\kappa$ set as $\kappa$ is increased. Even with relatively small values of $\kappa$, we find a sizeable set of stable points that, no matter how targets are combined, fall in the top-$\kappa$ set.}
        \label{fig:multi_target_ambiguity}
\end{figure*}

\subsection{Multi-target versus individual-target multiplicity}
Finally, we compare the latitude afforded by across-target multiplicity to that for within-target multiplicity.
To do this we return to the original dataset released by Obermeyer et al. and work with a subset of the features and examples for computational efficiency, as described in Appendix~\ref{sec:dataset_details}.

We evaluate predictive multiplicity by computing the single-target top-$\kappa$ ambiguity for each choice of target variable Eq~\eqref{Eq::Ambiguity-top-2} by running $\verb|FlipTopKMIP|$ for different error tolerances $\epsilon$. This allows us to determine the proportion of top-$\kappa$ points that can be flipped.
The results are shown in Fig.~\ref{fig:healthcare_single_vs_muli_target_ambiguity}, with each color corresponding to a choice of target variable.
From this, we see that single-target ambiguity rises quickly with $\epsilon$ and then plateaus. This is a result of the resource constrained predictive allocation setting: at a certain level, $\epsilon$, the $\epsilon$-Rashomon set contains the "flipping" model for each of the flippable points, so further increasing $\epsilon$ does not further increase ambiguity. 
We observe that the total cost variable has the highest ambiguity, slightly above 2\%, whereas active chronic conditions plateau just above 1\%. 

We compute the multi-target top-$\kappa$ ambiguity \eqref{Eq::Ambiguity-alpha} by running $\verb|FlipTopKMultiMIP|$ across the three different target variables.
This results in a multi-target ambiguity of nearly 5\%, as indicated by the black ``x'' and dashed horizontal line in Fig.~\ref{fig:healthcare_single_vs_muli_target_ambiguity}.
From this we see that the across-target multiplicity is substantially higher than the within-target multiplicity---a much higher proportion of points can be flipped into the top-$\kappa$ set by re-weighting predictions for the different targets than by entertaining slightly sub-optimal model fits for the individual targets. 

Finally, in Fig.~\ref{fig:healthcare_multi_target_stability} we look at the complement of ambiguity, examining the set of stable points that remain in the top-$\kappa$ set over all possible index models, as defined in  ~\textsection\ref{sec:stable-points}. Specifically, we plot the percent of stable points in the top-$\kappa$ set as we increase $\kappa$ to cover more of the entire dataset. Interestingly we see that the fraction of stable points grows rapidly with $\kappa$. 
For instance, when combining individual targets under an index model to select the top 10\% of highest risk patients, more than half of the selected patients are the same \textit{regardless of which index model is used}. 
This invariance lends confidence to the decision to prioritize such patients.

\section{Concluding Remarks}
In this paper, we introduced frameworks for assessing the level of multiplicity present in predictive allocation tasks in both the single-target and multi-target setting.  
First, we show how to measure multiplicity for a given target variable in settings where decision makers face constraints that limit the total number of people who can receive a scarce resource. 
Second, we show that when faced with a choice of multiple target variables, practitioners can develop index models that address fairness concerns by re-weighting and combining predictions for each target.
Our empirical results show that both of these methods are effective for narrowing racial disparities in selection rates in allocating patients to a high-risk coordinated care management program. Notably, we find that the latitude afforded by re-weighting predictions across target variables is substantially larger than the flexibility provided by leveraging within-target multiplicity.  This may represent a ``best of both worlds'' solution: we are able to fit separate models for predicting outcomes that may be interesting to model in their own right, but we can also combine the predictions from these models to allocate resources more equitably.

\begin{acks}
We thank David C. Parkes and members of the FATE research group at Microsoft Research for feedback and helpful discussions. JWD is supported by a Ford Foundation Pre-doctoral Fellowship and the NSF  Graduate Research Fellowship Program under Grant No. DGE1745303. Any opinions, findings, and conclusions or recommendations expressed in this material are those of the author(s) and do not necessarily reflect the views of the NSF.
\end{acks}

\clearpage

\bibliographystyle{ACM-Reference-Format}
\bibliography{new_references.bib}

\appendix
%
%
\clearpage

\onecolumn
\appendix

\section{Technical details for ranked single target multiplicity}
\label{sec:app-single-target-details}

This section presents the derivations and proofs of technical results appearing in ~\textsection\ref{sec::multiplicity}.  

\subsection{$\epsilon$-Rashomon set for linear regression}

In the \textsection\ref{sec::multiplicity} on single-target resource constrained predictive multiplicity, we repeatedly use the fact that, for orthonormal design matrices, $X$, the $\epsilon$-Rashomon set is given by
\[
\epsset{w_0} = \{w \in \R^{d+1} : \|w_0 - w\|_2^2 \le \epsilon\}.
\]
Here we provide a quick proof, which follows as a Corollary of Theorem 10 in \citet{Semenova2019a}.  

\begin{proof}
    Unpenalized linear regression is a special case of ridge regression 
    \[
\min_w L(w;\lambda) = \min_w (y - Xw)^T(y - Xw) + \lambda\|w\|_2^2,
    \]
    with $\lambda = 0$.  Part 1 of Theorem 10 of \citet{Semenova2019a} shows that the $\epsilon$-Rashomon set for ridge regression is,
    \[
    \epsset{w_0; X, \lambda} = \{w : (w - w_0)^T \left(X^TX + \lambda I_{d+1}\right)(w - w_0) \le \epsilon\}.
    \]
    For orthonormal designs, $X^TX = I_{d+1}$. This, combined with taking $\lambda = 0$ to recover the unpenalized linear regression setting gives the stated result.
\end{proof}

\subsection{$M$ bound in constraints \eqref{alt2a} and  \eqref{alt2b}}
\label{sec:app-M-bound-single-target}

To ensure that $\predIndic_{(\twopt > \onept)}$ whenever $(x_{\twopt} - x_\onept)^T w = \hat y (x_\twopt) - \hat y (x_\onept) > 0$ we need to choose $M$ so that 
\[
M \ge \hat y (x_\twopt) - \hat y (x_\onept) \quad \forall \twopt, \onept, \ \text{ and } \quad  \forall w \in \epsset{w_0}
\]

\begin{proposition}
\[
\hat y (x_\twopt) - \hat y (x_\onept) \le \left( \sqrt{\|w_0\|_2^2 + \epsilon} \right) \max_{i, j}\|x_j - x_i\|_2  \quad  \forall \twopt, \onept, \ \text{ and } \quad  \forall w \in \epsset{w_0}
\]
\end{proposition}

\begin{proof}
\[
\max_{w \in \epsset{w_0}} \hat y_\twopt - \hat y_\onept  
= \max_{w \in \epsset{w_0}} (x_\twopt - x_\onept)^T w
\]
By Cauchy-Schwartz,
\[
(x_\twopt - x_i)^T w \le \|x_\twopt - x_i\|_2 \|w\|_2 \le \|x_\twopt - x_i\|_2 \max_{w \in \epsset{w_0}} \|w\|_2.
\]
Noting that 
\[
\|w\|_2 = \sqrt{\|w_0 + (w - w_0)\|_2^2} 
\le \sqrt{\|w_0\|_2^2 + \|(w - w_0)\|_2^2} 
\le \sqrt{\|w_0\|_2^2 + \epsilon} \quad \forall w \in R_\epsilon(w_0),
\]
we therefore get that, 
\[
M_i = \max_{\twopt} \max_{w \in R_\epsilon} \hat y_{\twopt} - \hat y_{i}  
\le \left( \sqrt{\|w_0\|_2^2 + \epsilon} \right) \max_{ \twopt}\|x_\twopt - x_i\|_2 .
\]
Taking the maximum over all $\twopt$ gives the desired result,
\[
M = \max_{i, \twopt} \max_{w \in \epsset{w_0}} \hat y_{\twopt} - \hat Y_{i}  
\le \left( \sqrt{\|w_0\|_2^2 + \epsilon} \right) \max_{i, \twopt}\|x_\twopt - x_i\|_2 .
\]

\end{proof}

Note that the proof shows that one can set $M_i$ differently for each point $\onept$ we are aiming to flip in the given run of the MIP.

\subsection{Identifying certifiably (un)flippable points without solving a MIP}
\label{sec:app-unflippable}

\begin{proof}[Proof of Proposition \ref{prop:pred-gap-bound}]
\begin{align*}
    \Delta_{i, \twopt}(\wb) &= x_{i}^T\wb - x_\twopt^T\wb = (x_{i} - x_i )^\top \wb \\
    &= (x_{i} - x_i )^\top \wb + (x_{i} - x_\twopt )^\top \hat{\wb} - (x_{i} - x_i )^\top \hat{\wb} \\
    &= (x_{i} - x_\twopt )^\top (\wb - \hat{\wb} )+ (x_{i} - x_\twopt )^\top \hat{\wb} 
\end{align*}

By Cauchy-Schwartz,
\begin{align*}
\left|(x_{i} - x_\twopt)^T(\wb - \hat \wb)\right| & \le \|x_{i} - x_\twopt\|_2 \| \wb - \hat\wb\|_2 \\
&\le \sqrt{\epsilon}\|x_{i} - x_\twopt\|_2,
\end{align*}
where in the second step we use the fact that $\wb \in \epsset{w_0}$.

Thus $\forall \wb \in \epsset{w_0}$,
\[
\Delta_{i, \twopt}(\wb) \le \Delta_{i, \twopt}(\hat \wb) + \sqrt{\epsilon}\|x_{i} - x_\twopt\|_2 = B(i, \twopt; \epsilon).
\]
So if $B(i, \twopt; \epsilon) <0$, we have $\Delta_{i, \twopt}(\wb) < 0 \ \forall w \in \epsset{w_0}$.

\end{proof}

\begin{proof}[Proof of Corollary \ref{cor:unflippable}]
    $\{i': B(i, \twopt; \epsilon) < 0)\} \ge \kappa$ means that there are at least $\topK$ points for which \\ \mbox{$\Delta_{i, \twopt}(\wb) < 0 \, \forall \wb \in \epsset{w_0}$}, so $i$ cannot be in the top-$\topK$ set for any model in the $\epsilon$-Rashomon set.  
\end{proof}

\begin{proof}[Proof of Proposition \ref{prop:max-yhat}]
Let 
\[
w^* = w_0 + \sqrt{\epsilon}\frac{x_i}{\|x_i\|_2}.
\]
We will show that $\forall w \in \epsset{w_0}$, $\hat y_i(w) \le \hat y_i(w^*)$.  By construction,
\[
\hat y_i(w^*) = x^Tw_0 + \sqrt{\epsilon} \|x_i\|_2.
\]
By Cauchy-Schwartz, for any $w \in\epsset{w_0}$
\begin{align*}
\hat y_i(w) 
&= x_i^Tw_0 + x_i^T(w - w_0) \\
&\le x_i^Tw_0 +  \|x_i\|_2\|w - w_0\|_2 \\
&\le x_i^Tw_0 +  \sqrt{\epsilon}\|x_i\|_2 \\
&= \hat y_i(w^*)
\end{align*}
\end{proof}

\section{Technical details for multi-target multiplicity and fairness}

\subsection{Equivalence of index model and index variable approaches for linear models}
\label{sec:app-im_iv_equiv_proof}

\begin{proof}[Proof of Proposition \ref{prop:equiv-im-iv-linear}]
Starting with the index variable definition, we get that
\[
\hat y_{IV}^{(\alpha)} = M_X \yobs{\alpha} 
= M_X \left( \sum_{k = 1}^K \alpha_k  \yobs{k} \right)
= \sum_{k = 1}^K \alpha_k M_X \yobs{k}
= \sum_{k = 1}^K \alpha_k \hat y^{(k)} 
= \hat y_{IM}^{(\alpha)}
\]
\end{proof}

\subsection{Identifying certifiably (un)flippable points in the multi-target setting without solving a MIP}
\label{sec:app-unflippable-IM}

\begin{proposition}[Prediction gap bound for index models.]
    Let $\Delta_{\onept, \twopt}(\alpha) := \hat y_{IM}(x_i;\alpha) - \hat y_{IM}(x_{i'};\alpha)$ to be the prediction gap between instances $\twopt$ and $\onept$ under combining parameters $\alpha$.  For all $\onept, \twopt$ and $\alpha, \alpha' \in \Simplex{K}$,
    \[
    \Delta_{\onept,\twopt}(\alpha) \le  \Delta_{\onept,\twopt}(\alpha') + \sum_{k = 1}^K |\yhat{k}(x_i) - \yhat{k}(x_{i'}) |
    =: B_{IM}(\onept, \twopt; \alpha).
    \] 
\label{prop:pred-gap-bound-IM}
\end{proposition}

\begin{proof}
For any two instances $x_i, x_{\twopt} \in \X$ and combining parameter vectors $\alpha, \alpha' \in \Simplex{K}$,
\begin{align*}
\Delta_{\onept, \twopt}(\alpha) 
&= \Delta_{\onept, \twopt}(\alpha') + \sum_{k = 1}^K (\alpha_k - \alpha'_k)\left(\yhat{k}(x_i) - \yhat{k}(x_{\twopt})\right) \\
&\le \Delta_{\onept, \twopt}(\alpha') + \sum_{k = 1}^K \left| \yhat{k}(x_i) - \yhat{k}(x_{\twopt}) \right| \\
&= B_{IM}(i, \twopt; \alpha)
\end{align*}
\end{proof}

\begin{corollary}[Points that cannot appear in top-$\topK$ set for any index model] Suppose $\onept$ is not in the top-$\topK$ for an index model with parameter $\tilde \alpha$; 
 i.e., $Top_{(\onept,\tilde \alpha,\topK)} = 0$. If
$
\#\{\twopt: B(\onept, \twopt; \tilde \alpha) < 0\} \ge \topK,
$
then $Top_{(\onept,w,\topK)} = 0  \ \ \forall \alpha \in \Simplex{K} $.
\label{cor:unflippable-IM}
\end{corollary}

\begin{proof}
    $\{i': B_{IM}(i, \twopt; \tilde\alpha) < 0)\} \ge \kappa$ means that there are at least $\topK$ points, $\twopt$, for which \\ \mbox{$\Delta_{i, \twopt}(\alpha) < 0 \, \forall \alpha \in \Simplex{K}$}, so $i$ cannot be in the top-$\topK$ set for any index model.  
\end{proof}

Proposition~\ref{prop:pred-gap-bound-IM} establishes a bound on the gap between the predicted values of any two points \textit{for all} $\alpha \in \Simplex{K}$ in terms of the prediction gap under \textit{any one} choice of combining parameters $\alpha$.  Corollary~\ref{cor:unflippable-IM} then allows us to determine when $\onept$ cannot be in the top-$\topK$ of \textit{any} index model $\alpha \in \Simplex{K}$ based on the prediction gap for a \textit{given} $\tilde \alpha$.  

\begin{proposition}[Prediction maximizing index model] The predicted value of point $i$ is maximized at $\alpha^* \in \Simplex{K}$ where $\alpha^*_{k^*} = 1$ for $k^* = \argmax_k \yhat{k}(x_i) $ and $\alpha^*_{k} = 0$ for $k \neq k^*$.  
\label{prop:max-yhat-IM}
\end{proposition}

\begin{proof}
    \[
    \max_{\alpha \in \Simplex{K}} \hat y_{IM}^{(\alpha)}(x_i) 
    = \max_{\alpha \in \Simplex{K}} \sum_{k = 1}^K \alpha_k \yhat{k}
    \le \max_k \yhat{k}  \sum_{k = 1}^K \alpha_k = \max_k \yhat{k},
    \]
    which is achieved at the stated value of the combining parameter vector, $\alpha^*$.
\end{proof}

Proposition~\ref{prop:max-yhat-IM} provides a candidate $\alpha$ for which a given point may be in the index model's top-$\topK$.  Note that this result does not preclude the possibility that $Top_{(i, \alpha^*, \topK)} = 0$ while also $Top_{(i, \alpha', \topK)} = 1$ for some other $\alpha' \in \Simplex{K}$.  This result suggests the simple strategy of first identifying points whose top-$\topK$ decision varies between the single-target prediction models $\yhat{k}$.

\section{Dataset details}
\label{sec:dataset_details}
We remove several co-linear features from the original Obermeyer et al. dataset so that models can be fit with OLS regression (instead of regularized regression). Specifically, we remove features whose variable name matches the following regular expression:\footnote{See data dictionary for names and descriptions of variables here: \url{https://gitlab.com/labsysmed/dissecting-bias/-/blob/master/data/data_dictionary.md}}
\begin{quote}
  \begin{verbatim}
    (gagne_sum_tm1|normal_tm1|esr_.*-low_tm1|crp_(min|mean|max).*_tm1|ghba1c_.*-low_tm1)
  \end{verbatim}
\end{quote}
This eliminates the sum of active illnesses (which are listed as individual binary features in the dataset) as well as one-hot encoded indicators for individual test results that have low/normal/high levels.

For \textsection\ref{sec:index_model_intuition} we use a smaller subset of features, for computational efficiency, taking only features that match the following regular expression:
\begin{quote} 
  \begin{verbatim}
    (gagne_sum_tm1|hypertension_elixhauser_tm1|^dem_|cost.*tm1)
  \end{verbatim}
\end{quote}
This takes the count of total illnesses in the previous time period instead of individually coded illnesses along with demographics and cost in the previous time period.

\section{Maximizing fairness with multiple targets}
\label{sec:index_model_diagram}
\begin{figure}[H]
    \centering
    \includegraphics[width=\textwidth]{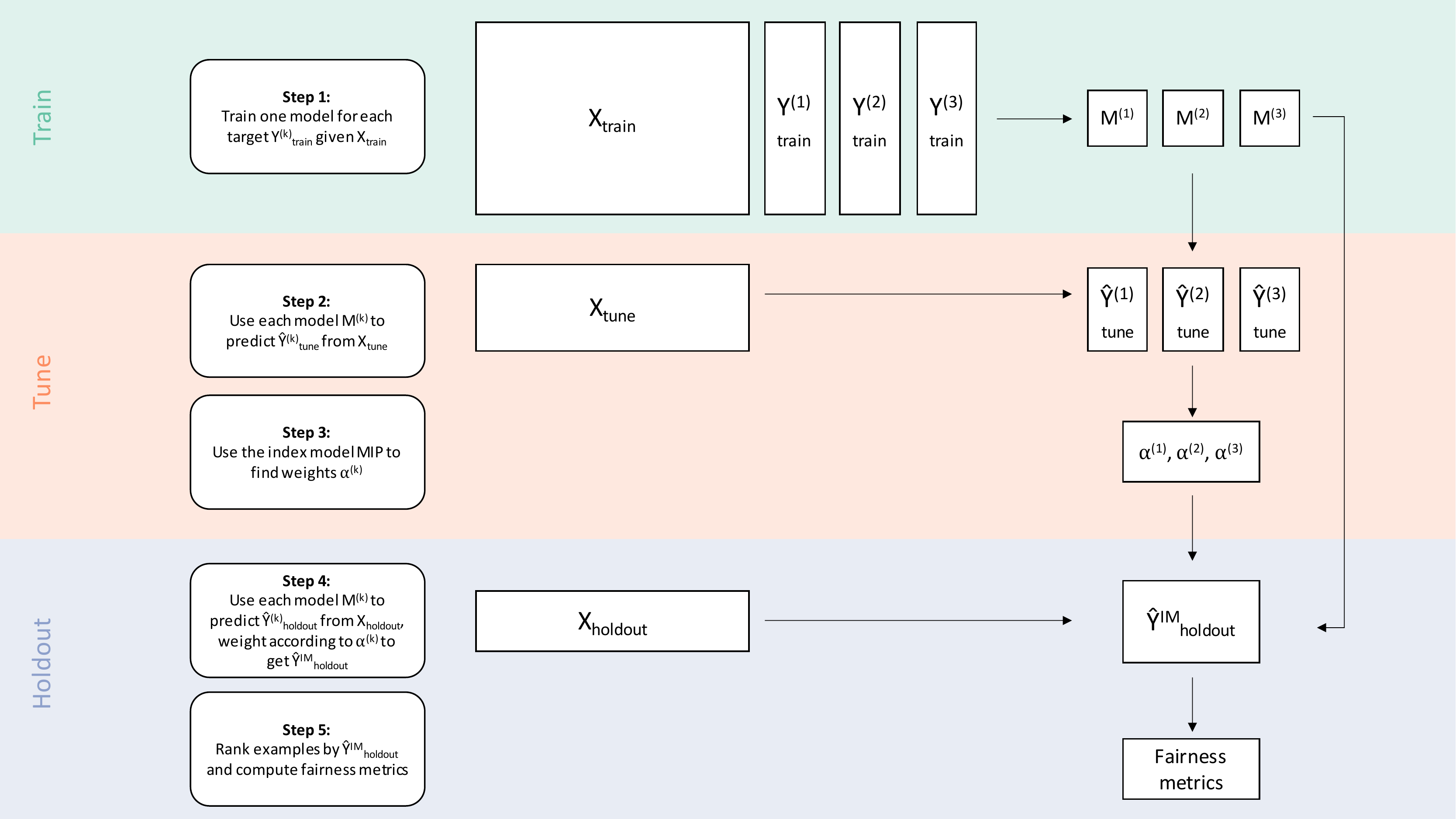}
    \caption{Workflow for maximizing fairness with multiple targets. Training data is used to fit separate models for each target variable. In the tune phase, the fitted models are used to forecast each target variable, and the index model MIP is run to find a fairness-maximizing weighted combination of the targets. Finally, in the holdout phase a separate dataset is used to calculate the weight index model predictions, from which fairness metrics are computed.}
    \label{fig:index_model_diagram}
\end{figure}

\end{document}